\documentclass{article}

\usepackage[preprint]{neurips_2019}

\usepackage[utf8]{inputenc} 
\usepackage[T1]{fontenc}    
\usepackage{hyperref}       
\usepackage{url}            
\usepackage{booktabs}       
\usepackage{amsfonts}       
\usepackage{nicefrac}       
\usepackage{microtype}      
\usepackage{color}
\usepackage{float}

\usepackage{algorithmic}
 
\usepackage{hyperref}

\usepackage{wrapfig}
\usepackage{subcaption}
\usepackage{subcaption} 
\usepackage[ruled]{algorithm2e}
\usepackage{amssymb,amsmath}
\usepackage{amsthm}
\usepackage{bm}
\usepackage{dsfont}

\def\grad{\nabla}

\def\expectation{\mathbb{E}}

\def\prob{P}

\def\IND{\mathds{1}}

\newcommand{\ent}{\mathbb{H}}
\newtheorem{theorem}{Theorem}

\newtheorem{corollary}{Corollary}[theorem]

\usepackage{graphicx}
\graphicspath{{Plots/}}

\title{Entropy Regularization with Discounted Future State Distribution in Policy Gradient Methods}

\author{%
Riashat Islam\\
McGill University, Mila\\
School of Computer Science\\
riashat.islam@mail.mcgill.ca\\
\And
Raihan Seraj\\
McGill Univeristy\\
raihan.seraj@mail.mcgill.ca\\
\AND
Pierre-Luc Bacon\\
Stanford Univeristy\\
plbacon@cs.stanford.edu\\
\And
Doina Precup\\
McGill University, Mila\\
School of Computer Science\\
dprecup@cs.mcgill.ca\\
}

\begin{document}

\maketitle

\begin{abstract}

The policy gradient theorem is defined based on an objective with respect to the initial distribution over states. In the discounted case, this results in policies that are optimal for one distribution over initial states, but may not be uniformly optimal for others, no matter where the agent starts from. Furthermore, to obtain unbiased gradient estimates, the starting point of the policy gradient estimator requires sampling states from a normalized discounted weighting of states. However, the difficulty of estimating the normalized discounted weighting of states, or the stationary state distribution, is quite well-known. Additionally, the large sample complexity of policy gradient methods is often attributed to insufficient exploration, and to remedy this, it is often assumed that the restart distribution provides sufficient exploration in these algorithms. In this work, we propose exploration in policy gradient methods based on maximizing entropy of the discounted future state distribution. The key contribution of our work includes providing a practically feasible algorithm to estimate the normalized discounted weighting of states, i.e, the \textit{discounted future state distribution}. We propose that exploration can be achieved by entropy regularization with the discounted state distribution in policy gradients, where a metric for maximal coverage of the state space can be based on the entropy of the induced state distribution. The proposed approach can be considered as a three time-scale algorithm and under some mild technical conditions, we prove its convergence to a locally optimal policy. Experimentally, we demonstrate usefulness of regularization with the discounted future state distribution in terms of increased state space coverage and faster learning on a range of complex tasks.

\end{abstract}

\section{Introduction}

Exploration in policy optimization methods is often tied to exploring in the policy parameter space. This is primarily achieved by adding noise to the gradient when following stochastic gradient ascent. More explicit forms of exploration within the state and action space include policy entropy regularization. This promotes stochasticity in policies, thereby preventing  premature convergence to deterministic policies \citep{mhiha2c, eq_pgq}. Such regularization schemes play the role of smoothing out the optimization landscape in non-convex policy optimization problems \citep{understanding_entropy}. Deep reinforcement learning algorithms have had enormous success with entropy regularized policies, commonly known as maximum entropy RL framework \citep{Ziebart}. These approaches ensure exploration in the action space, which indirectly contributes to exploration in the state space, but do not explicitly address the issue of state space exploration. This leads us to the question : \textit{how do we regularize policies to obtain maximal coverage in the state space?}

One of the metrics to measure coverage in state space is the entropy of the \textit{discounted} future state distribution, as proposed in~\citep{kakade_entropy}. In their work, they prove that using the entropy of discounted future state distribution as a reward function, we can achieve improved coverage of the state space. Drawing inspiration from this idea, and to provide a practically feasible construct, we first propose an approach to estimate the discounted future state distribution. We then provide an approach for efficient exploration in policy gradient methods, to reduce sample complexity, by regularizing policy optimization based on the entropy of the discounted future state distribution. The implication of this is that the policy gradient algorithm yields policies that improve state space coverage by maximizing the entropy of the discounted future state distribution induced by those policies as an auxiliary regularized objective. This distribution takes into account when various states are visited in addition to which states are visited. The main contribution of our work is to provide a practically feasible way to estimate the discounted future state distribution with a density estimator. Furthermore, we show that regularizing policy gradients with the entropy of this distribution can improve exploration. To the best of our knowledge, there are no previous works that provide a practical realization for estimating and regularizing with the entropy of the discounted state distribution.

It is worthwhile to note that the estimation of the discounted/stationary state distribution is not readily achievable in practice. This is because the stationary distribution requires an estimate based on rollouts, as in value function estimates, under a given policy $\pi$. In contrast, the discounted state distribution requires estimation of discounted occupancy measures for the various states. Since the discounted occupancy measure is purely a theoretical construct, it is not possible to sample from this distribution using rollouts. In order to use this as an entropy regularizer, we also need the discounted or stationary distributions to be explicitly dependent on the policy parameters, which is not straightforward in practice. 

To address this, we estimate the state distribution by separately training a density estimator based on sampled states in the rollout. The crucial step here is that, we use a density estimator that is explicitly a function of the policy parameters $\theta$. In other words, our density estimator takes as input, the parameters $\theta$ of the policy itself (for instance, weights of a policy neural network) through which we now obtain an estimate of $p_{\theta}: \theta \mapsto \Delta(S)$, where $p_\theta(s)$ is the occupancy probability (discounted or otherwise) of state $s$. We use a variational inference based density estimator, which can be trained to maximize a variational lower bound to the the log-likelihood of $p_{\theta}(s)$. As a result, we can obtain an estimation of $d_{\pi_{\theta}}$ since in case of stationary distributions, we have $\log p_{\theta}(s) = \log d_{\pi_{\theta}}(s)$. Estimation of $d_{\pi_{\theta}}$ under any policy $\pi_{\theta}$ requires collecting a large number of samples from the rollout. Instead of this, we can use ideas from multi-scale stochastic algorithms to learn this in an online manner. Hence, we require a separate time-scale for training the density estimator, in addition to learning the policy and value functions in policy gradient based approaches. We formally state and prove the corresponding three time-scale algorithm. The key contributions of our work are as follows : 

\begin{itemize}
    \item  We provide a practically feasible algorithm based on neural state density estimation, for entropy regularization with the \textit{discounted} future state distributions in policy gradient based methods. This can be adapted for both episodic and infinite horizon environments, by switching between stationary and discounted state distributions with a $(1-\gamma) \gamma^{t}$ weighted importance sampling correction
    
    \item For regularization with the state distribution, we require the estimated state distribution to be directly inlfuenced by the policy parameters $\theta$. We achieve this by learning the density estimator directly as a function of policy parameters $\theta$, denoted by $d_{\pi_{\theta}}$, such that the policy gradient update can be regularized with $\nabla_{\theta} \ent(d_{\pi_{\theta}})$
    
    \item Our approach requires estimating the state distribution, in addition to \textit{any} existing actor-critic or policy optimization method, for state distribution entropy regularization. This leads to a three time scale algorithm based on our approach. We provide convergence for a \textit{three-time-scale} algorithm and show that under sufficiently mild conditions, this approach can converge to the optimal solution.

    \item We demonstrate the usefulness of entropy regularization with the discounted state distribution on a wide range of deep reinforcement learning tasks, and discuss the usefulness of this approach compared to entropy regularization with stochastic policies, as commonly used in practice.
    
\end{itemize}

\section{Preliminaries}

We consider the standard RL framework, where an agent acts in an environment that can be modelled as a Markov decision process (MDP). Formally we define an MDP as a tuple $(\mathcal{S}, \mathcal{A}, P, r)$, where\footnote{For ease of exposition we assume that all spaces being considered are finite. The approach that we propose extends in a straightforward manner to continuous spaces using standard measurability conditions.}: $\mathcal{S}$ is the state space, $P: \mathcal{S} \times \mathcal{A} \to \Delta(\mathcal{S})$ is the transition probability matrix/kernel of that outputs a distribution over states given a state-action pair, $r: \mathcal{S} \times \mathcal{A} \to \mathbb{R}$ is the reward function that maps state-action pairs to real numbers. In this work, we focus on policy based methods. We use parametrized policies $\pi_\theta$, where the parameters are $\theta \in \Theta$ and $\Theta$ is a compact, convex set. The performance of any such policy is given by $    J(\theta) = \expectation_{A_t \sim \pi_{\theta}, S_{t+1} \sim P(S_t, A_t)} \Bigl[\sum_{t=0}^{\infty} \gamma^{t} r(S_t, A_t)  \Bigm | S_0 \Bigr]$, where $\gamma \in (0,1)$ is the discount factor and $S_0$ is the initial state with a distribution $P_{0}$. The objective of the agent is to find a parameter $\theta^*$ that maximizes this performance i.e., $\theta^* \in \arg\max J(\theta)$. The usual approach to find such a $\theta^*$ is to use a stochastic gradient ascent based iteration: $\theta_{k+1} = \theta_k + \alpha_k \grad_\theta{J(\theta_k)}$, where $\grad_\theta{J(\theta)}$ can be obtained using the policy gradient theorem~\cite{sutton, dpg} or one of its variants. It has been shown in literature that this iteration converges almost surely to a local maximum of $J(\theta)$ under relatively mild technical conditions. 

Maximum entropy based objectives are often used in policy gradient methods to help with exploration by avoiding premature convergence of stochastic policies to deterministic policies. This is often termed as entropy regularization which augments the rewards with an entropy weighted term given by $r_t = r_{t} + \lambda H( \pi(.|s_t )$. One of the common approaches is to use an entropy regularized objective, given by $    \tilde{J}(\theta) = \expectation_{\pi_{\theta}} \Bigl[\sum_{t=0}^{\infty} \gamma^{t} r(S_t, A_t) + \lambda \ent (\pi_{\theta}(\cdot \mid S_t) \Bigm | S_0 \Bigr]$, where $\lambda \in \mathbb{R}$ is a hyperparameter and $\ent (\pi_{\theta}(\cdot \mid S_t)$ is the entropy of the policy function (for stochastic policies).

\section{State Distribution in Policy Gradient Methods}

Policy gradient theorem ~\citep{sutton} for the starting state formulation are given for an initial state distribution $\mathbf \alpha$, where the exact solution for the discounted objective is given by $J{\theta} = \mathbf \alpha^{T} v_{\theta} = \mathbf \alpha^{T} (\mathbf I - \gamma P_{\theta})^{-1} r_{\theta}$. In \citep{sutton}, this is often known as the \textit{discounted weighting of states} defined by $d_{\alpha, \gamma, \pi}^{T} = \mathbf \alpha^{T} (\mathbf I - \gamma P_{\theta})^{-1}$, where in the average reward case this reaches a stationary distribution implying that the process is independent of the initial states. However, the discounted weighting of states is not a distribution, or a stationary distribution in itself, since the rows of the matrix $(\mathbf I - \gamma P_{\theta})^{-1}$ do not sum to 1. The normalized version of this is therefore often considered, commonly known as the \textit{discounted future state distribution} ~\citep{Kakade2003OnTS} or the discounted state distribution ~\citep{Thomas14}. Detailed analysis of the significnace of the state distribution in policy gradient methods is further given in [Bacon, 2018]. 

\begin{equation}
\label{eq:gamma_discounted_future_state}
    \Bar{d}_{\alpha, \gamma, \pi} = (1 - \gamma) d_{\alpha, \gamma, \pi} = ( 1 - \gamma) \mathbf\alpha^{T} (\mathbf I - \gamma \mathbf P_{\pi})^{-1} = ( 1 - \gamma) \mathbf \alpha^{T} \sum_{t=0}^{\infty} \gamma^{t} P_{\pi} (s_t = s)
\end{equation}

Given an infinite horizon MDP, and a stationary policy $\pi(a,s)$, equation \ref{eq:gamma_discounted_future_state} is the $\gamma$ discounted future state distribution, ie, the normalized version for the discounted weighting of states. We can draw samples from this distribution, by simulating $\pi$ and accepting each state as the sample with a probability $(1 - \gamma)$. With the discounted future state distribution, the equivalent policy gradient objective can therefore be given by $J({\theta}) = \Bar{ \mathbf  d}_{\alpha, \gamma, \theta}^{T} \mathbf r_{\theta}$. In practice, we want to express the policy gradient theorem with an expectation that we can estimate by sampling,  but since the discounted weighting of states $d_{\alpha, \gamma, \pi}$ is not a distribution over states, we often use the normalized counterpart of the discounted weighting of states $\Bar{d}_{\alpha, \gamma, \pi}$ and correct the policy gradient with a factor of $\frac{1}{(1 - \gamma)}$. 

\begin{equation}
    \nabla_{\theta} J(\theta) = \frac{1}{(1 - \gamma)} \expectation{_{\Bar{d}_{\alpha, \gamma, \theta}, a \sim \pi_{\theta} } } [ \nabla_{\theta} \log \pi_{\theta}(a,s) Q_{\pi_{\theta}}(s,a) ]
\end{equation}

However, since the policy gradient objective is defined with respect to an initial distribution over states, the resulting policy is not \textit{optimal} over the entire state space, ie, not \textit{uniformly optimal}, but are rather optimal for one distribution over the initial states but may not be optimal for a different starting state distribution. This often leads to the large sample complexity of policy gradient methods \citep{Kakade2003OnTS} where a large number of samples may be required for obtaining good policies. The lack of exploration in policy gradient methods may often lead to large sample complexity to obtain accurate estimates of the gradient direction. It is often assumed that the restart, or starting state distribution in policy gradient method provides sufficient exploration. In this work, we tackle the exploration problem in policy gradient methods by explicitly using the entropy of the discounted future state distribution. We show that even for the starting state formulation of policy gradients, we can construct the normalized discounted future state distribution, where instead of sampling from this distribution (which is hard in practice, since sampling requires discounting with $(1-\gamma)$, we instead regularize policy optimization with the entropy $\ent((1 - \gamma) \mathbf d_{\alpha, \gamma, \theta})$

\section{Entropy Regularization with Discounted Future State Distribution}

The key idea behind our approach is to use regularization with the entropy of the state distribution in policy gradient methods. In policy optimization based methods, the state coverage, or the various times different states are visited can be estimated from the state distribution induced by the policy. This is often called the discounted (future) state distribution, or the normalized discounted weighting of states. In this work, our objective is to promote exloration in policy gradient methods by using the entropy of the discounted future state distribution $d_{\alpha, \gamma, \pi}$ (which we will denote as $d_{\pi_{\theta}}$) where $\mathbf \alpha$ is the distribution over the initial states and to explicitly highlight that this distribution is dependent on the changes in the policy parameters $\theta$, and we propoe a practically feasible algorithm for estimating and regularizing policy gradient methods with the  discounted state distribution for exploration and reducing sample complexity. 


We propose the following state distribution entropy regularized policy gradient objective: ${\tilde{J}(\theta) = \expectation_{\pi_{\theta}} \Bigl[\sum_{t=0}^{\infty} \gamma^{t} r(S_t, A_t) \Bigm | S_0 \Bigr] + \lambda \ent (d_{\pi_{\theta}})}$, where $d_{\pi_{\theta}}$ is the discounted state distribution induced by the policy $\pi$. We can estimate $\nabla_{\theta} {J}(\theta)$ while using stochastic policies from~\citep{sutton} or deterministic policies from~\citep{dpg}. The regularized policy gradient objective is:
\begin{equation}
\label{eq:modified_entobj_gradient}
    \nabla_{\theta} \tilde{J}(\theta) = \nabla_{\theta}{J}(\theta) +\lambda \nabla_{\theta} \ent(d_{\pi_{\theta}})
\end{equation}

\paragraph{Entropy of the \textit{discounted} state distribution $\ent (d_{\alpha, \gamma, \pi})$ : } The discounted state distribution $d_{\alpha, \gamma, \pi_\theta}$ can be computed as:
\begin{equation}\label{eq:def_disc}
    \setlength{\abovedisplayskip}{3pt}
    d_{\alpha, \gamma, \pi_\theta}(s) = (1-\gamma) \mathbf \alpha^{T} \sum_{t=0}^{\infty}\gamma^t\prob(S_t = s), \quad \forall s \in \mathcal{S}
\end{equation}
We note that this is a theoretical construct and we cannot sample from this distribution, since it would require sampling each state with a probability $(1-\gamma)$ such that the accepted state is then distributed according to $d_{\alpha, \gamma, \pi_{\theta}}$. However, we can modify the state distribution $p(s)$ to a weighted distribution $\tilde p(s)$ as follows. We estimate $p(s)$ from samples as:
\[\setlength{\abovedisplayskip}{3pt}
p(s) = \frac{1}{T}\sum_{t=0}^{T}\IND({S_t = s}),
\]
where the weight of each sample is $1/T$. To estimate $\tilde p(s)$, we use an importance sampling weighting of $(1-\gamma)\gamma^t$ to yield:
\begin{equation}\label{eq:tildep}
\setlength{\abovedisplayskip}{3pt}
\tilde p(s) \stackrel{(a)}{=} \frac{(1-\gamma)}{T}\sum_{t=0}^{T}\gamma^t\IND({S_t = s}) \stackrel{(b)}{=} (1-\gamma)\sum_{t=0}^T(\gamma^t\prob(S_t = s \mid S_0)) \stackrel{(c)}{\approx} d_{\gamma, \pi_\theta}(s),
\end{equation}
where $(a)$ follows from the importance sampling approach, $(b)$ follows from the fact that $\frac{\IND({S_t = s})}{T} = \prob(S_t = s \mid S_o)$ and $(c)$ follows from~\eqref{eq:def_disc} and the approximation is due to the finite truncation of the infinite horizon trajectory. Note that due to this finite truncation, our estimate of $d_{\gamma, \pi_\theta}$ will be sub-stochastic. Therefore, we can estimate the entropy of this distribution as:
\begin{equation}\label{eq:disc_ent}
    \ent (d_{\alpha, \gamma, \pi_\theta}) \approx -\frac{1}{T}\sum_{t=0}^T\log \tilde p(S_t).
\end{equation}

\paragraph{Entropy of the \textit{stationary} state distribution $\ent (d_{1, \pi_{\theta}})$ : } For the average reward case with infinite horizon MDPs, we can similiarly compute the entropy of the \textit{stationary} state distribution. The stationary distribution $d_{1, \pi_\theta}$ is a solution of the following fixed point equation satisfying
\begin{equation}
    d_{1, \pi_\theta} = P_{\pi_\theta}^\intercal d_{1, \pi_\theta},
\end{equation}
where $P_{\pi_\theta}$ is the transition probability matrix corresponding to policy $\pi_\theta$. In practice, this is the long term state distribution under policy $\pi_\theta$, which is denoted as $p(s)$. In infinite horizon problems, the stationary state distribution is indicative of the majority of the states visited under the policy. We expect the stationary state distribution to change slowly, as we adapt the policy parameters (especially for a stochastic policy). Hence, we assume that the states have mixed, as we learn the policy over several iterations. In practice, instead of adding a mixing time specifically, we can use different time-scales for learning the policy and estimating the stationary state distribution. The entropy of the stationary state distribution can therefore be computed as :  
\begin{equation}\label{eq:stationary_ent}
\setlength{\abovedisplayskip}{3pt}
\ent (d_{1, \pi_\theta}) \stackrel{(a)}{=} -\sum_{s \in \mathcal{S}}d_{1, \pi_\theta}(s) \log(d_{1, \pi_\theta}(s)) \stackrel{(b)}{\approx} - \frac{1}{T} \sum_{t=0}^{T} \log d_{1, \pi_\theta}(S_t) \stackrel{(c)}{=} - \frac{1}{T} \sum_{t=0}^{T} \log p(S_t), \end{equation}
where $T$ is a finite number of time-steps after which an infinite horizon episode can be truncated due to discounting. In deriving~\eqref{eq:stationary_ent}, $(a)$ follows from the definition of entropy, $(b)$ follows by assuming ergodicity, which allows us to replace an expectation over the state space with an expectation over time under all policies. The approximation here is due to the finite truncation of the infinite horizon to $T$. Step $(c)$ follows from the density estimation procedure.

\paragraph{\textit{Stationary} Distributions and \textit{Discounted} Future State Distributions:} 
For environments where the stationary distribution exists, for all policies, it is more natural to use the entropy of the \textit{stationary} state distribution as the regularizer. This is because we are interested in visiting all states in the state space, and not necessarily concerned with at what point in time the state is visited. However, there may be environments where the stationary distribution does not exist or environments in which the stationary distribution collapses to a unit measure on a single state. Episodic environments are examples of the latter where the support of the stationary state distribution only includes the terminal state. In such environments, it is more natural to use the \textit{discounted} state distribution (normalized discounted occupancy measure).

\subsection{Estimating the entropy of discounted future state distribution:} 

In practice, we use a neural density estimator for estimating the discounted state distribution, based on the states induced by the policy $\pi_{\theta}$. The training samples for the density estimator is obtained by rolling out trajectories under the policy $\pi_{\theta}$. We train a variational inference based density estimator (similar to a variational auto-encoder) to maximize variational lower bound $\log p(s)$, where for the discounted case, we denote this as $\log \tilde p(s)$, as given in ~\eqref{eq:tildep}~and~\eqref{eq:disc_ent}. We therefore obtain an approximation to the entropy of discounted future state distribution which can be used in the modified policy gradient objective, where for the \textit{discounted} case, with stochastic policies~\citep{sutton}, we have 

\begin{equation}
\label{eq:obj_discounted}
\setlength{\abovedisplayskip}{3pt}
\tilde{J}(\theta) = \expectation_{\pi_{\theta}} \Bigl[\sum_{t=0}^{\infty} \gamma^{t} r(S_t, A_t) - \lambda \log  d_{\alpha, \gamma, \pi_{\theta}}(s_t) \Bigm | S_0 \Bigr]
\setlength{\belowdisplayskip}{3pt}
\end{equation}

The objective in the \textit{stationary} case can be obtained by substituting the $d_{\alpha, \gamma, \pi_{\theta}}(s_t)$ with $d_{1, \pi_{\theta}}(s_t)$ in~\eqref{eq:obj_discounted}. The neural density estimator is independently parameterized by $\phi$, and is a function that maps the policy parameters $\theta$ to a state distribution. The loss function for this density estimator is the KL divergence between $ \texttt{KL}(q_{\phi}(Z \mid \theta) || p(Z|\theta))$. The training objective for our density estimator in the \textit{stationary} case is given by :
\begin{equation}
\label{eq:vae_lower_bound}
    \mathcal{L_{\gamma}}(\phi, \theta) = (1 - \gamma) \gamma^{k} \expectation_{q_{\phi}(Z| \theta)} \big[\log p_{\phi}(S|\theta)  \big] - KL\big(q_{\phi}(Z|\theta) || p(\theta)\big)
\end{equation}

Equation~\eqref{eq:vae_lower_bound} gives the expression for the loss function for training the state density estimator (which is the variational inference lower bound loss for estimating $log(p(s))$,i.e.,ELBO~\cite{DBLP:journals/corr/KingmaW13}. Here $\theta$ are the parameters of the policy network $\pi_{\theta}$, $\phi$ are the parameters of the density estimator. The novelty of our approach is that the density estimator takes as input the parameters of the policy network directly (similar to hypernetworks~\citep{krueger2017bayesian, ha2016hypernetworks}). The encoder then maps the policy parameters $\theta$ into the latent space $Z$ given by $q_{\phi}(Z\mid \theta)$ with a Gaussian prior over the policy parameters $\theta$. During implementation we feed the parameters of the last two layers of the policy network, assuming the previous layers extract the relevant state features and the last two layers map the obtained features to a distribution over actions. Hence $\theta$ only comprises of the weights of these last two layers ensuring computation tractability. We take this approach since the discounted future state distribution is a function of the policy parameters $\theta$.

Our overall gradient objective with the regularized update is therefore given by
$\tilde{J}(\theta) = \expectation_{\pi_{\theta}} \Bigl[\sum_{t=0}^{\infty} \gamma^{t} r(S_t, A_t) - \lambda  \mathcal{L_{\gamma}}(\phi, \theta) \Bigr]$, where $\mathcal{L_{\gamma}}(\phi, \theta)$ directly depends on the policy parameters $\theta$. This gives the regularized policy gradient update with the entropy of the discounted future state distribution, for stochastic policies as : 

\begin{equation}
\setlength{\abovedisplayskip}{3pt}
\label{eq:overall_gradient_update}
    \nabla_{\theta} \tilde{J}(\theta) = \expectation_{\pi_{\theta}} \Bigl[ \nabla_{\theta} \log \pi(A_t \mid S_t) Q^{\pi}(S_t, A_t) -\lambda \nabla_{\theta} \mathcal{L}_{\gamma}(\phi, \theta).  \Bigr],\text{where $\mathcal L_{\gamma}(\phi,\theta)=(1-\gamma)\gamma^t\mathcal L (\phi,\theta)$}
\end{equation}

\paragraph{\textit{Three Time-Scale Algorithm} :} Our overall gradient update in equation \ref{eq:overall_gradient_update} implies that before computing the policy gradient update, we need an estimate of the variational lower bound. This therefore requires implementing an added time-scale for updating $\mathcal{L_{\gamma}}(\phi, \theta)$, in addition to the existing two-time scales in the actor-critic algorithm \citep{konda2000actor}. Since these distributions affect only the update of the actor parameters, the learning rate for these distributions should be higher than the learning rate for the actor. Our approach requiring a separate density estimator, therefore leads to a three time-scale algorithm. As given in the appendix, we additionally formally state and prove a three time-scale algorithm, that provides convergence guarantees to locally optimal solutions under standard technical conditions \citep{borkar2009stochastic, kushner2003stochastic}.

\paragraph{Algorithm :} We summarize our main algorithm in Algorithm 1 given in the Appendix. We use a regular policy gradient or actor-critic algorithm to learn parameterized policy $\pi_{\theta}$ with any existing policy optimization method such as REINFORCE \citep{reinforce} or A2C \citep{mhiha2c}. Based on sampled rollouts, we only a require a separate density estimator described above. Our key step of the algorithm is to regularize the policy gradient update with 
$\mathcal{L}(\phi, \theta)$ for entropy regularization with the \textit{stationary} state distribution, or with $(1 - \gamma)\gamma^{t} \mathcal{L}(\phi, \theta)$ denoted by $\mathcal{L}_{\gamma}(\phi, \theta)$ for entropy regularization with the \textit{discounted} state distribution. 

\begin{algorithm}[!htb]
\caption{Entropy regularization with $\ent ( \hat{d_{\pi}} )$ }
\begin{algorithmic}
\label{algo:training}
\REQUIRE ~~ A policy $\pi_{\theta}$ and critic $Q_{\psi}(s,a)$ 
\REQUIRE ~~ A density estimator $p_{\phi}(s)$ and regularization weight $\lambda$
\FOR{episodes = 1 to \text{E}}
\STATE{Take action $a_{t}$,get reward $r_{t}$ and observe next state $s_{t+1}$}
\STATE{Store tuple ($s_t,a_t, r(s_{t+1}), s_{t+1},$) in $\mathcal{D}$}\\
        \uIf{mod(t,N)}{
        {Update critic parameters $\psi$ as policy evaluation}\\
        
        {Update density estimator $\phi$ to estimate $\log d_{\pi}(s)$ or $\log d_{\gamma, \pi}(s)$ by maximizing variational lower bound $\mathcal{L}(\phi, \theta)$ 
        }\\
        
        {Update policy parameters $\theta$ folowing \textit{any} policy gradient method according to
       $ \nabla_{\theta} \tilde{J}(\theta) = \Bigl[ \nabla_{\theta} \log \pi_{\theta}(A_t | S_t)$ $Q^{\psi}(A_t, S_t) - \lambda  \nabla_{\theta} \mathcal{L}_{\gamma}(\phi, \theta)    \Bigr]$}
        }
\ENDFOR 
\end{algorithmic}
\end{algorithm}

\section{Experiments}
In this section, we demonstrate our approach based on entropy regularization with the normalized discounted weighting of states, also known as the discounted future state distribution. 
We first verify our hypothesis to clarify our intuitions for the significance of the state distribution as a regularizer on simple toy domains. Our method can be applied on top of \textit{any} existing RL algorithm, and in our experiments we mostly use REINFORCE \citep{reinforce}, actor-critic \citep{konda2000actor} with non-linear function approximators and off-policy ACER algorithm \citep{acer}. We first demonstrate usefulness of our approach in terms of state space coverage, better exploration and faster learning in simple domains, and then extend to continuous control tasks to show significant performance improvements in standard domains. In all our experiments, we use $\gamma$-\textit{StateEnt} (or \textit{Discounted StateEnt}) for denoting entropy regularization with the discounted future state distribution, and \textit{StateEnt} for denoting the unnormalized counterpart of the discounted weighting of states.

\subsection{Entropy regularization in Exact Policy Gradients with $\ent(d_{\pi})$}

We first verify that entropy regularization with \textit{exact} discounted future state distribution $\ent (d_{\pi_{\theta}})$ can lead to benefits in policy optimization when used as a regularizer. We demonstrate this on three toy domains, varying the amount $\lambda$ of state distribution regularization, in the case where we can compute \emph{exact policy gradient} given by $J(\pi) = (I - \gamma P_{\pi})^{-1} R$. In all these examples, the optimal solution can be found with value iteration algorithm.

\begin{figure}[h]
\begin{center}
    \begin{subfigure}[b]{0.33\textwidth}
    \centering
        \includegraphics[width=\linewidth]{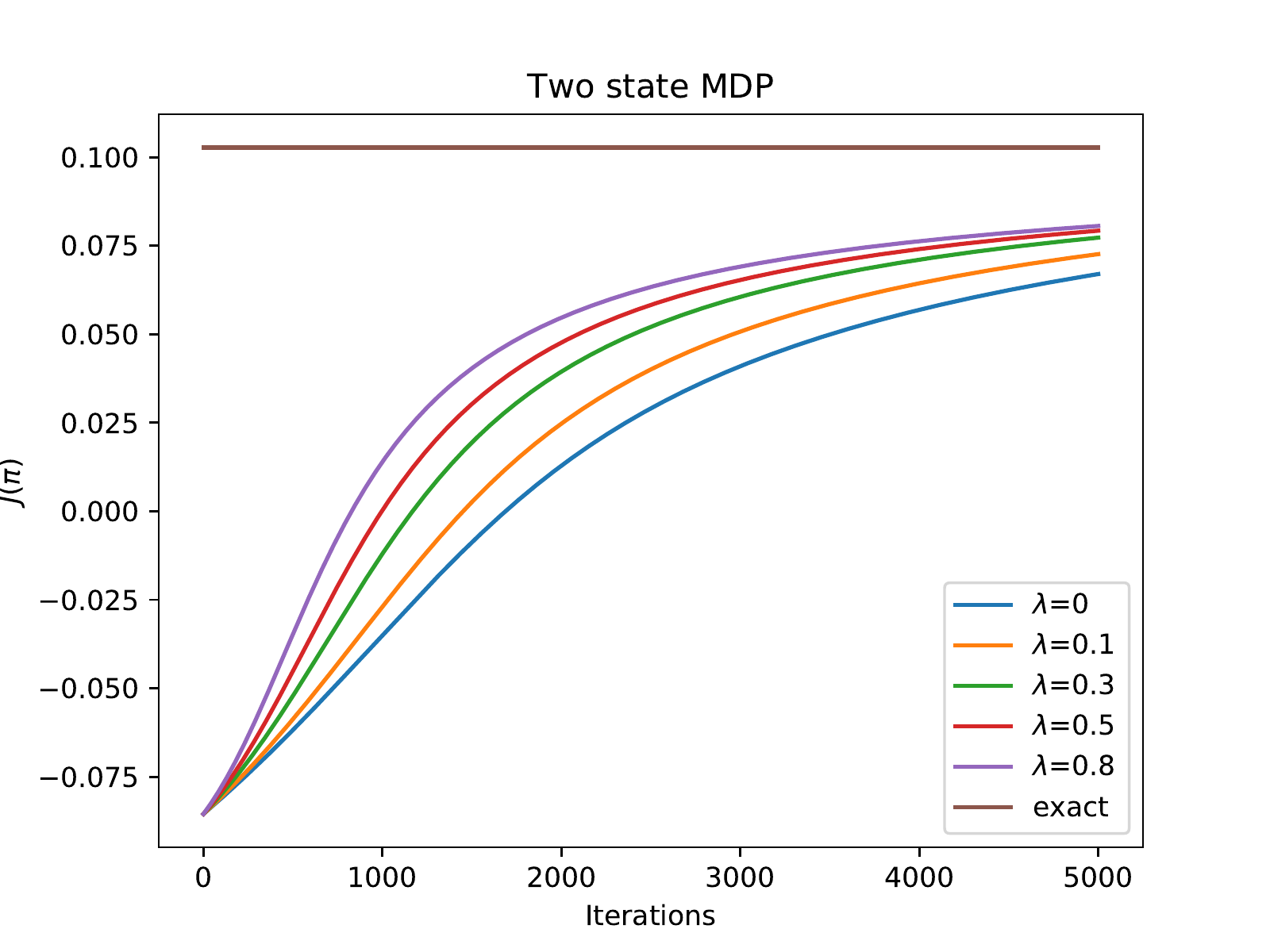}
        \label{fig:exact_pg}
        \caption{}
    \end{subfigure}\hfill
    \begin{subfigure}[b]{0.33\textwidth}
    \centering
        \includegraphics[width=\linewidth]{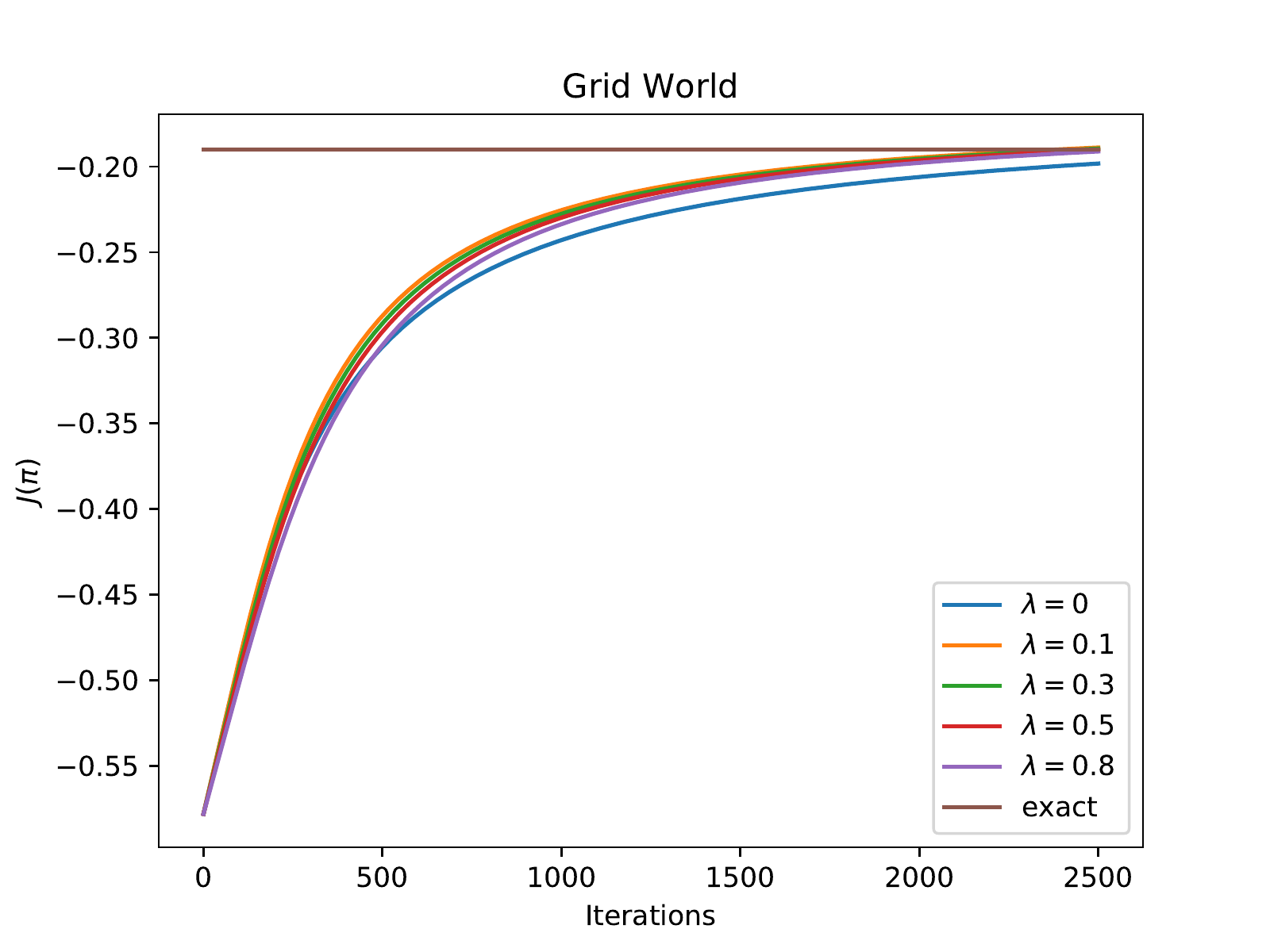}
        \label{fig:exact_pg2}
        \caption{}
    \end{subfigure}\hfill
     \begin{subfigure}[b]{0.33\textwidth}
    \centering
        \includegraphics[width=\linewidth]{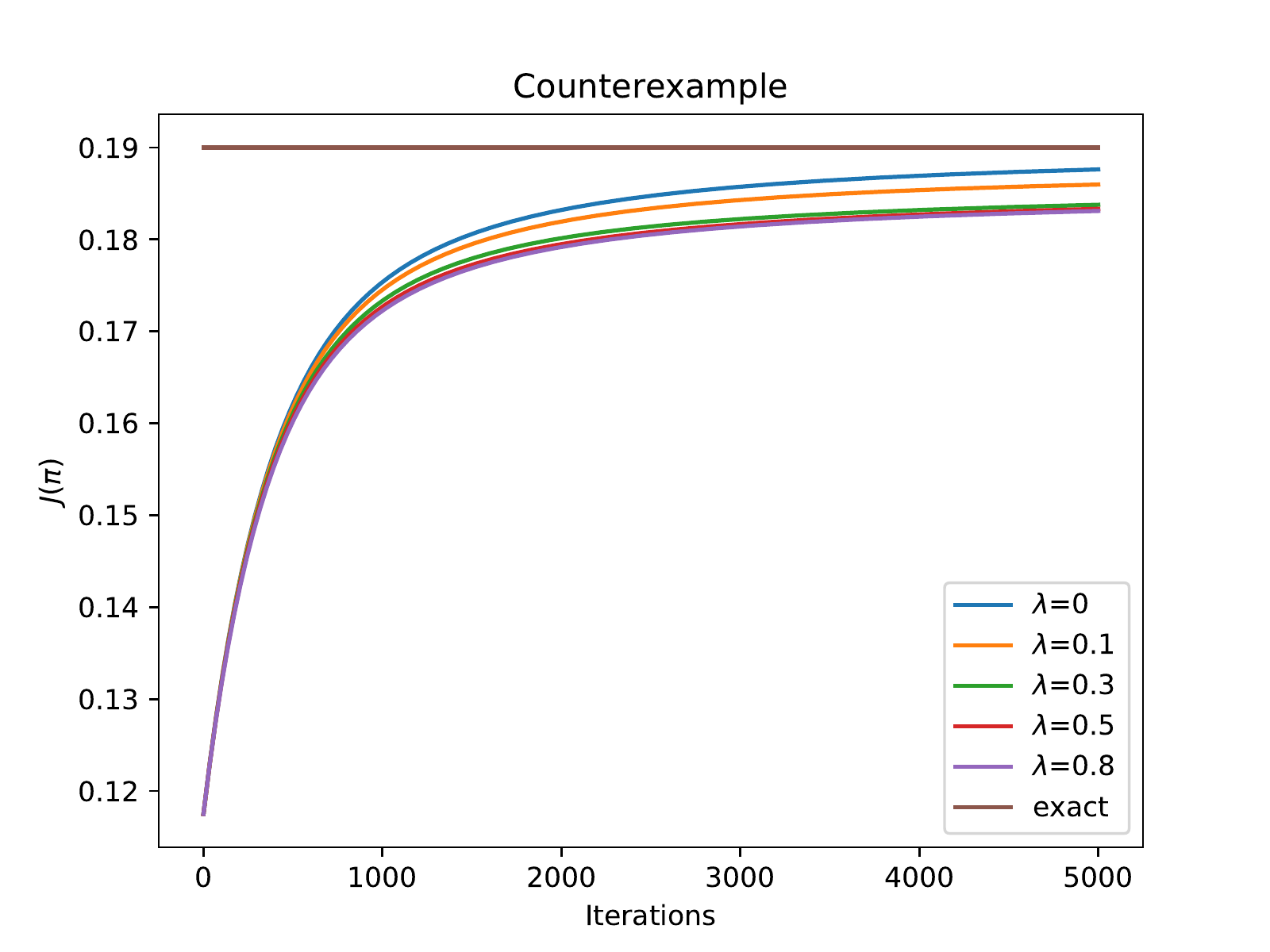}
        \label{fig:exact_pg3}
        \caption{}
    \end{subfigure}
    \caption{State distribution entropy regularized exact policy gradient can lead to a better converged solution on a simple two state MDP (taken from \citep{vf_polytope} (Figure (a)) and a discrete GridWorld domain (Figure (b)). The regularized objective has a faster convergence rate compared to the unregularized objective (with $\lambda=0.0$). Interestingly, in toy MDPs where there exists state aliasing, as in Figure (c) for MDP taken from the counterexample domain~\citep{imani2018off}, we find that state distribution entropy does not lead to significant improvements. This is an interesting result justifying that state space exploration may not necessarily be needed in all MDPs, especially when states are aliased. It is important to note that improved exploration does not necessarily mean faster convergence in policy based methods. For instance, in several goal based problems, the stationary distribution corresponding to an optimal policy is a delta distribution, which has the least entropy, whereas our regularizer tries to maximize this entropy. }
    \label{fig:twostateMDP_comp}
\end{center}    
\end{figure}


\subsection{Toy Domains}

Having verified our hypothesis in figure \ref{fig:twostateMDP_comp}, we now present our approach based on separately learning a density estimator for the state distribution, on tabular domains with actor-critic algorithms. For these tasks, we use a one hot encoded state representation with a one layer network for the policies, value functions and the state distribution estimator. We compare our results for both the \textit{stationary} and \textit{discounted} state distributions, with a baseline actor-critic (with $\lambda=0.0$ for the regularizer). Figure \ref{fig:toy_actor_critic} summarizes our results.

\begin{figure}[h]
\begin{center}
    \begin{subfigure}[b]{0.33\textwidth}
    \centering
        \includegraphics[width=\linewidth]{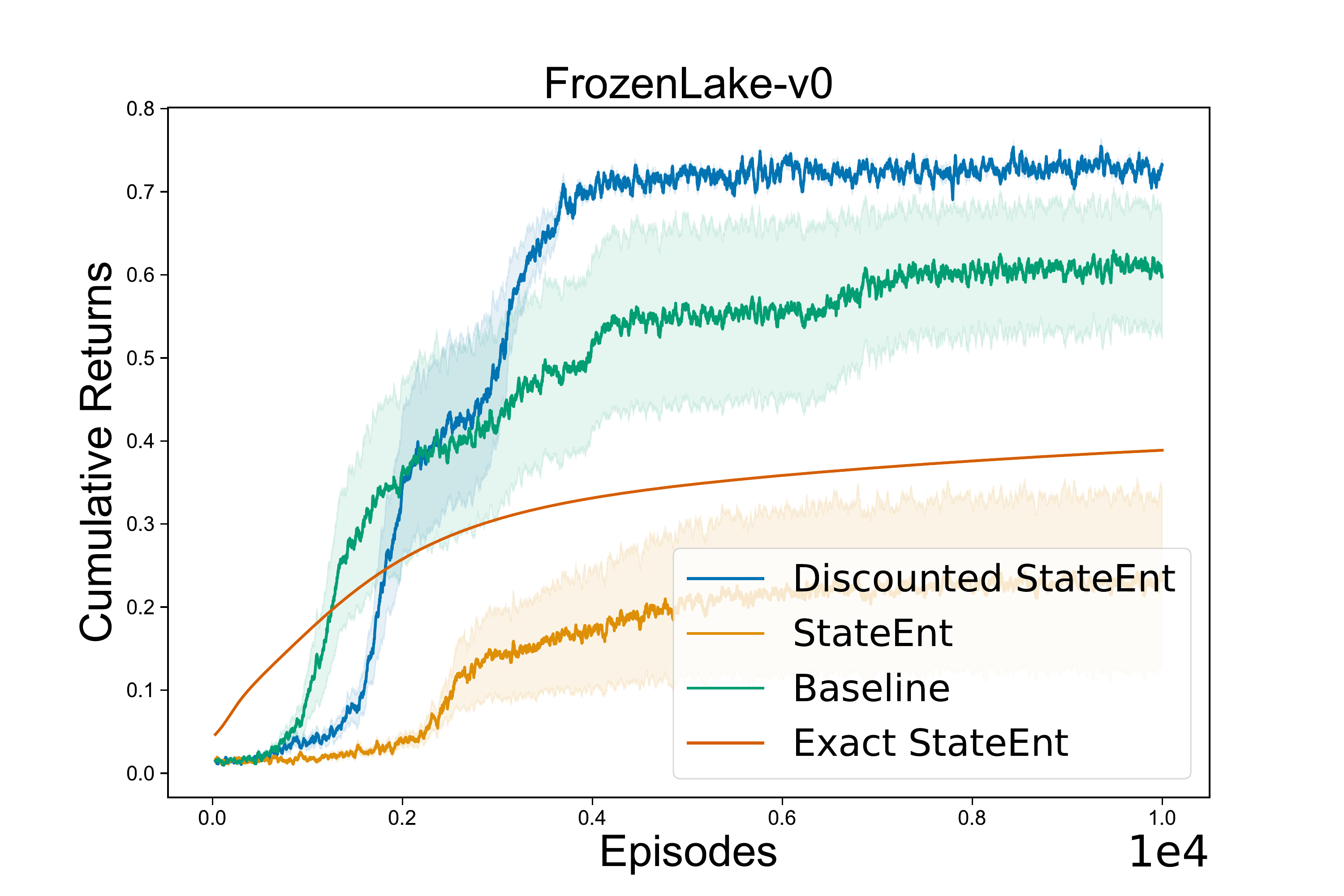}
        \label{fig:frozelake_comp}
    \end{subfigure}\hfill
    \begin{subfigure}[b]{0.33\textwidth}
    \centering
        \includegraphics[width=\linewidth]{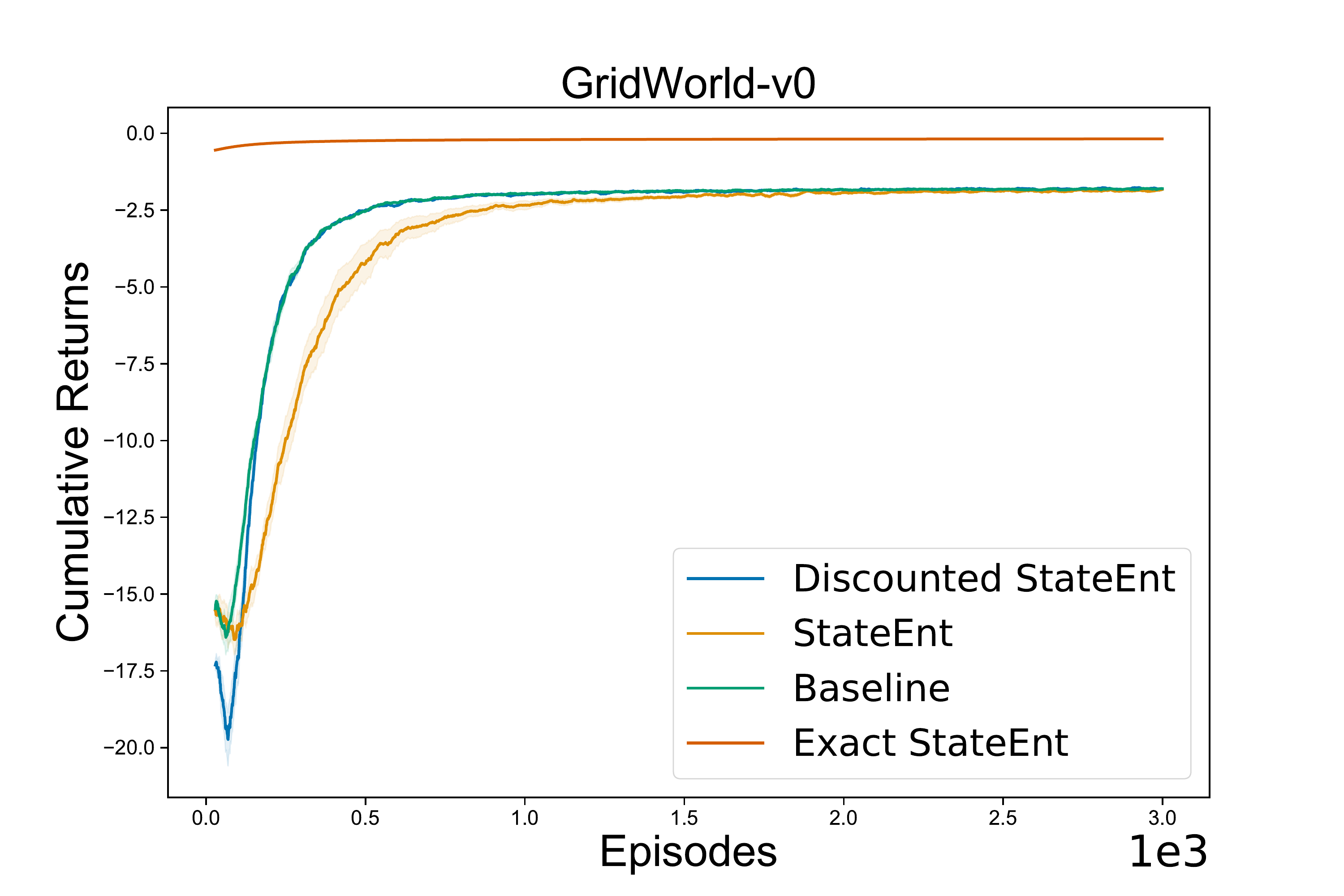}
        \label{fig:gridworld_comp}
    \end{subfigure}\hfill
    \begin{subfigure}[b]{0.33\textwidth}
    \centering
        \includegraphics[width=\linewidth]{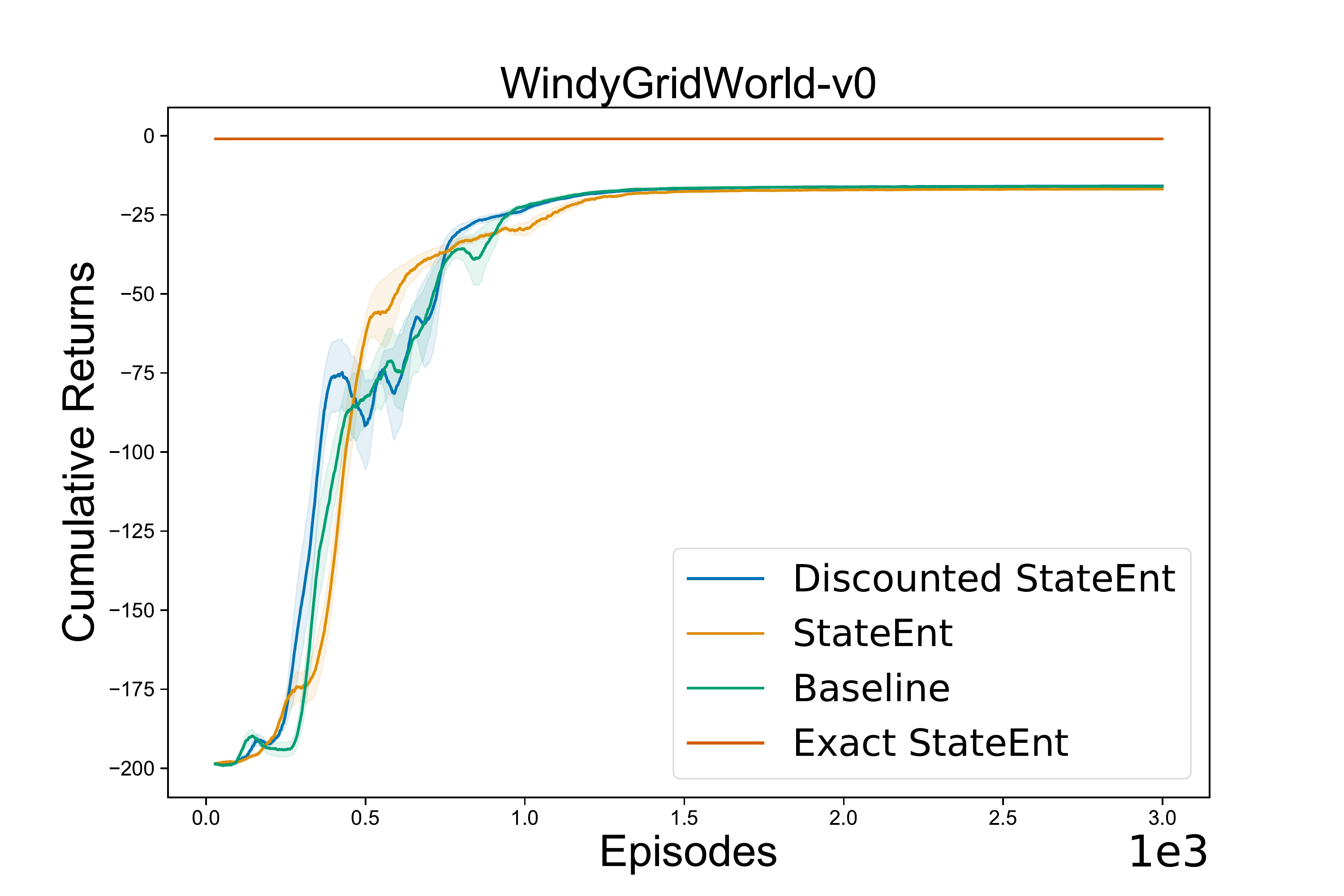}
        \label{fig:windygrid_comp}
    \end{subfigure}  

\caption{We show benefits of state distribution entropy regularization on toy domains, especially hard exploration tabular tasks such as FrozenLake. In all the tasks, we find that regularization with entropy of \textit{discounted} state distribution performs significantly better compared to baseline and regularization with the unnormalized state distribution. In all tasks, we use $\lambda=0.1$ for our methods.} 
\label{fig:toy_actor_critic}
\end{center}    
\end{figure}

\subsection{Complex Sparse Reward GridWorld Tasks}

We then demonstrate the usefulness of our approach, with entropy of stationary (denoted StateEnt) and discounted  (denoted $\gamma$ StateEnt) state distributions, on sparse reward complex gridworld domains. 
These are hard exploration tasks, where the agent needs to pass through slits and walls to reach the goal state (placed at the top right corner of the grid). We use REINFORCE \citep{reinforce} as the baseline algorithm, and for all comparisons use standard policy entropy regularization (denoted PolicyEnt for baseline).

\begin{figure}[h]
\begin{center}
    \begin{subfigure}[b]{0.33\textwidth}
    \centering
        \includegraphics[width=\linewidth]{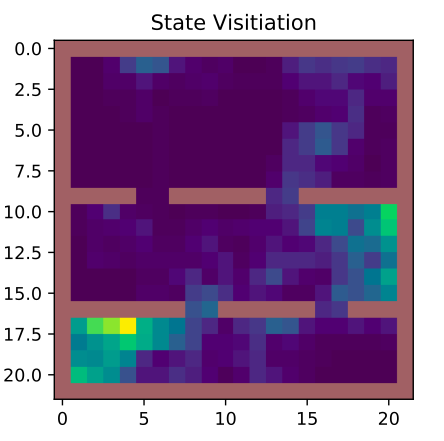}
        \caption{$\gamma$ StateEnt}
    \end{subfigure}\hfill
    \begin{subfigure}[b]{0.33\textwidth}
    \centering
        \includegraphics[width=\linewidth]{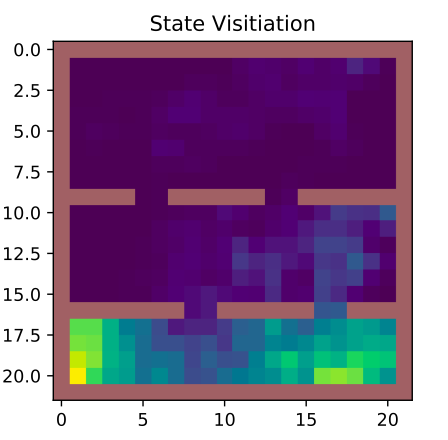}
        \caption{StateEnt}
    \end{subfigure}\hfill
    \begin{subfigure}[b]{0.33\textwidth}
    \centering
        \includegraphics[width=\linewidth]{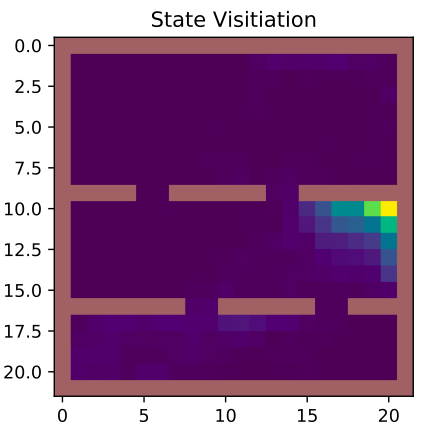}
        \caption{PolicyEnt}
    \end{subfigure}
    \caption{State space coverage on complex sparse reward double-slit-double-wall gridworld domains. Figure shows that regularization with the \textit{discounted} state distribution indeed has more useful effects in terms of exploration and state space coverage compared to regularization with policies. We also find that state space coverage is more with entropy of \textit{discounted} state distribution compared to \textit{stationary} state distribution. All state visitation heatmaps are shown after only $1000$ timesteps of initial training phase.}
    \label{fig:grid_reinforce}
\end{center}    
\end{figure}

\subsection{Simple Benchmark Tasks}

We extend our results with standard deep RL toy benchmark tasks, using on-policy Reinforce \citep{williams1991function} and off-policy ACER  \citep{acer} algorithms. Figure \ref{fig:toy_benchmark} shows performance improvements with discounted state distribution entropy regularization across all algorithms and tasks considered.

\begin{figure}[h]
\begin{center}
    \begin{subfigure}[b]{0.33\textwidth}
    \centering
        \includegraphics[width=\linewidth]{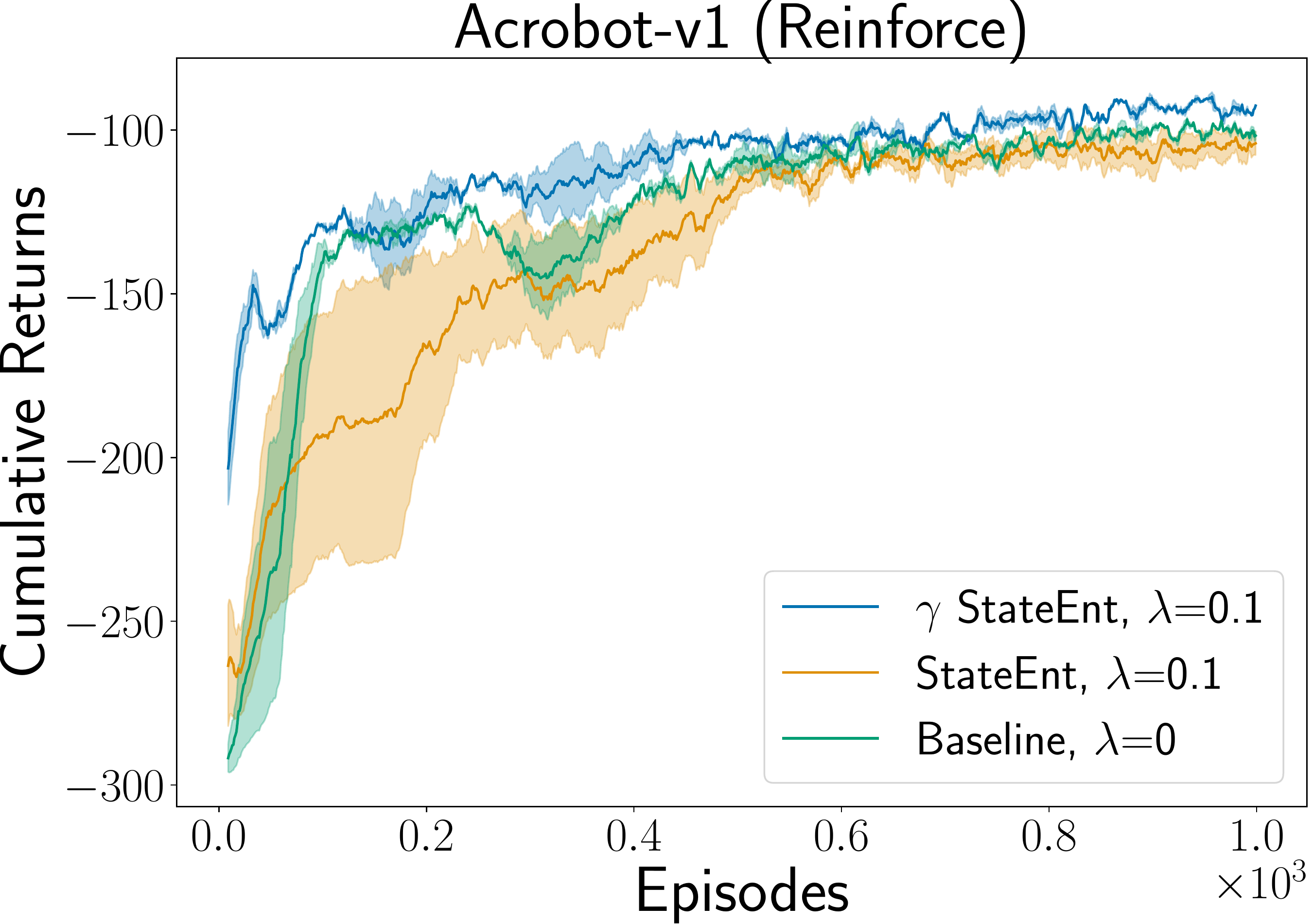}
    \end{subfigure}\hfill
    \begin{subfigure}[b]{0.33\textwidth}
    \centering
        \includegraphics[width=\linewidth]{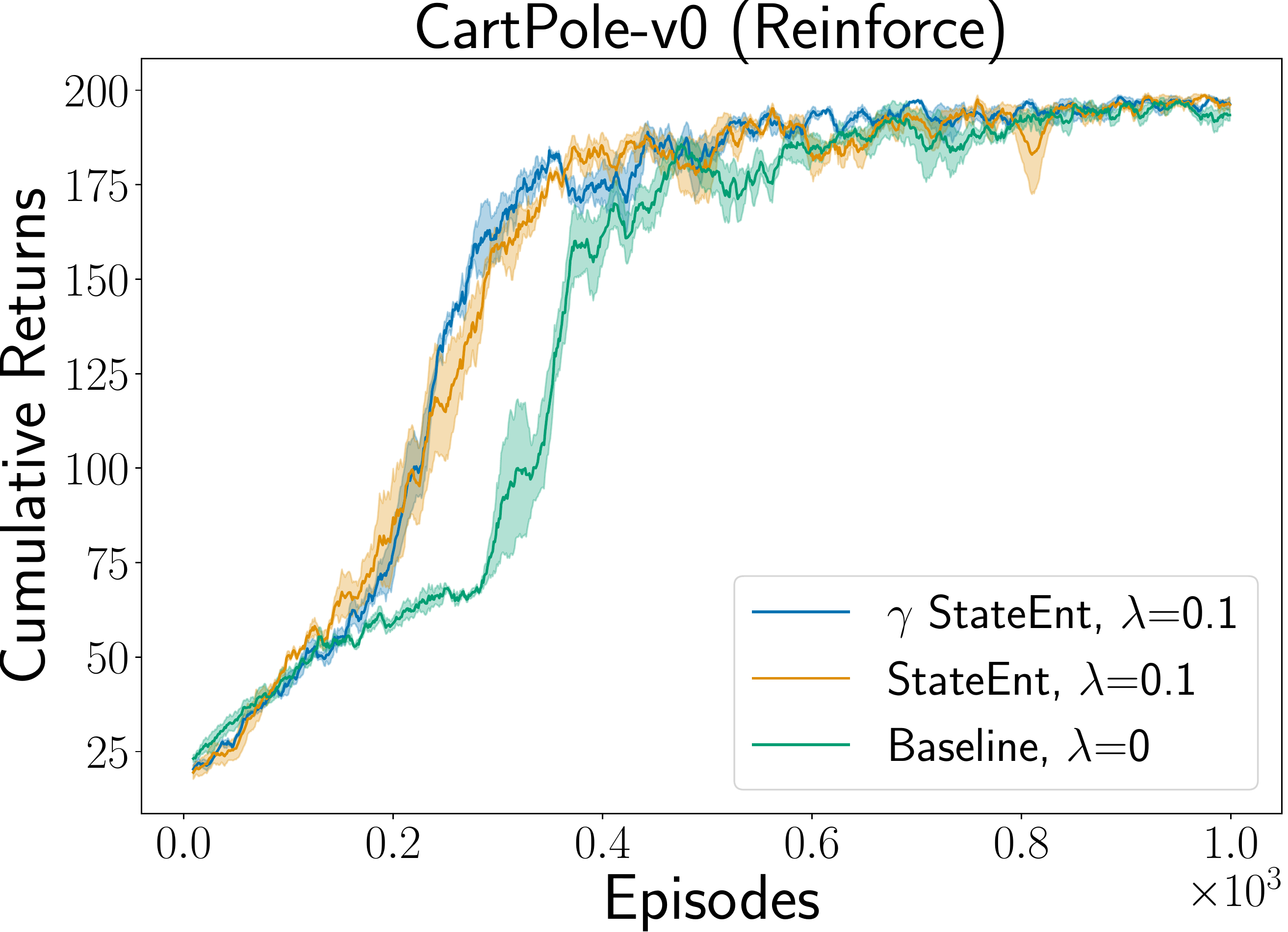}
    \end{subfigure}\hfill
    \begin{subfigure}[b]{0.33\textwidth}
    \centering
        \includegraphics[width=\linewidth]{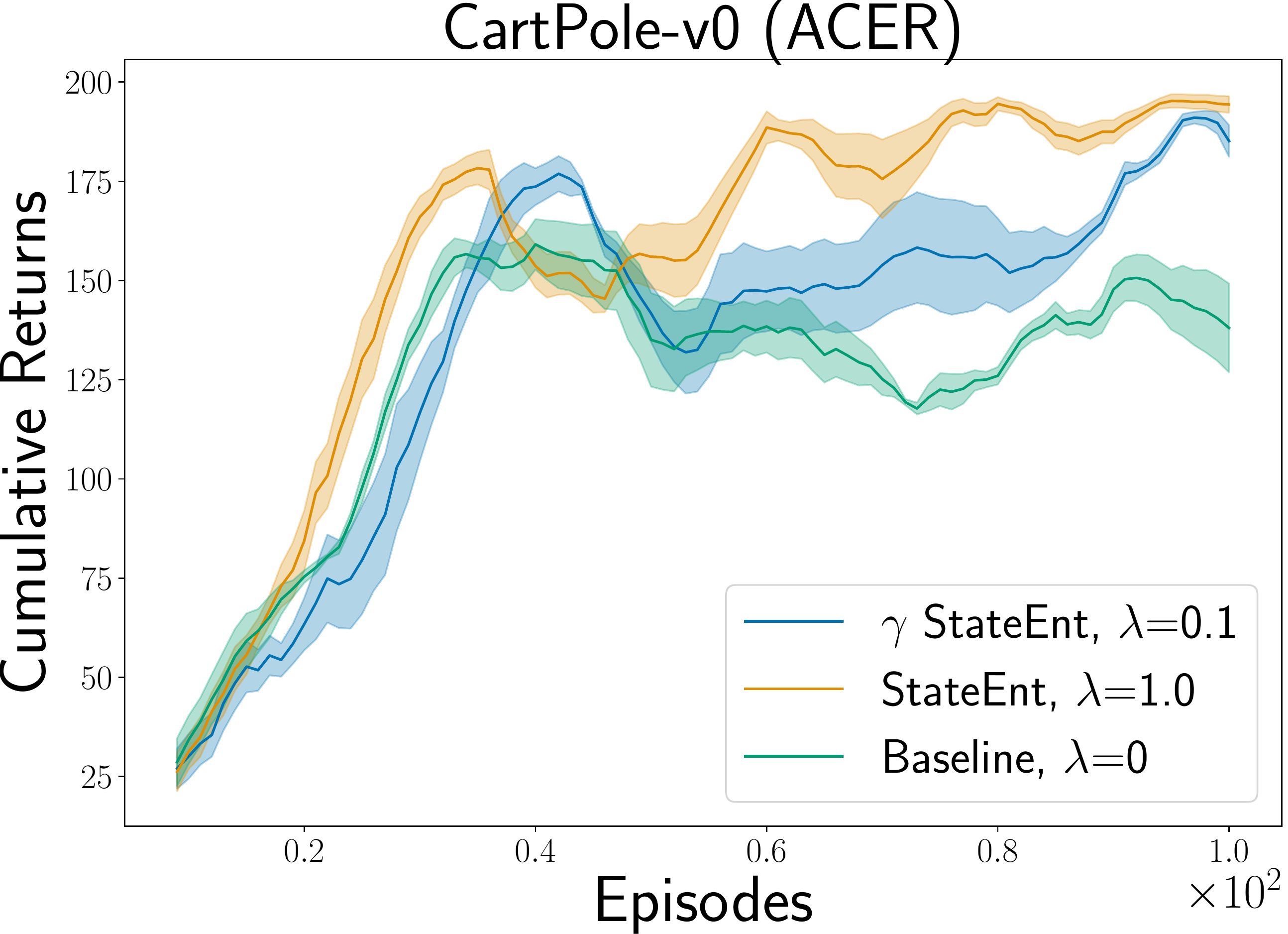}
    \end{subfigure}
    \caption{Performance improvements with state distribution entropy regularization on standard benchmark tasks. We find consistent improvements in performance with $\gamma$ StateEnt that uses entropy of discounted state distributions. Further experimental results and ablation studies are given in the Appendix.}
    \label{fig:toy_benchmark}
\end{center}    
\end{figure}

\subsection{Continuous Control Tasks}
Finally, we extend our proposed regularized policy gradient objective on standard continuous control Mujoco domains \citep{conf/iros/TodorovET12}. First, we examine the significance of the state distribution entropy regularizer in  DDPG algorithm~\citep{DDPG}. In DDPG, policy entropy regularization cannot be used due to existence of deterministic policies \citep{dpg}. In Figure \ref{fig:ddpg_control}, we show that by inducing policies to maximize state space coverage, we can enhance exploration that leads to significant improvements on standard benchmark tasks, especially in environments where exploration in the state space plays a key role (e.g HalfCheetah environment)

\begin{figure}[h]
\begin{center}
    \begin{subfigure}[b]{0.33\textwidth}
    \centering
        \includegraphics[width=\linewidth]{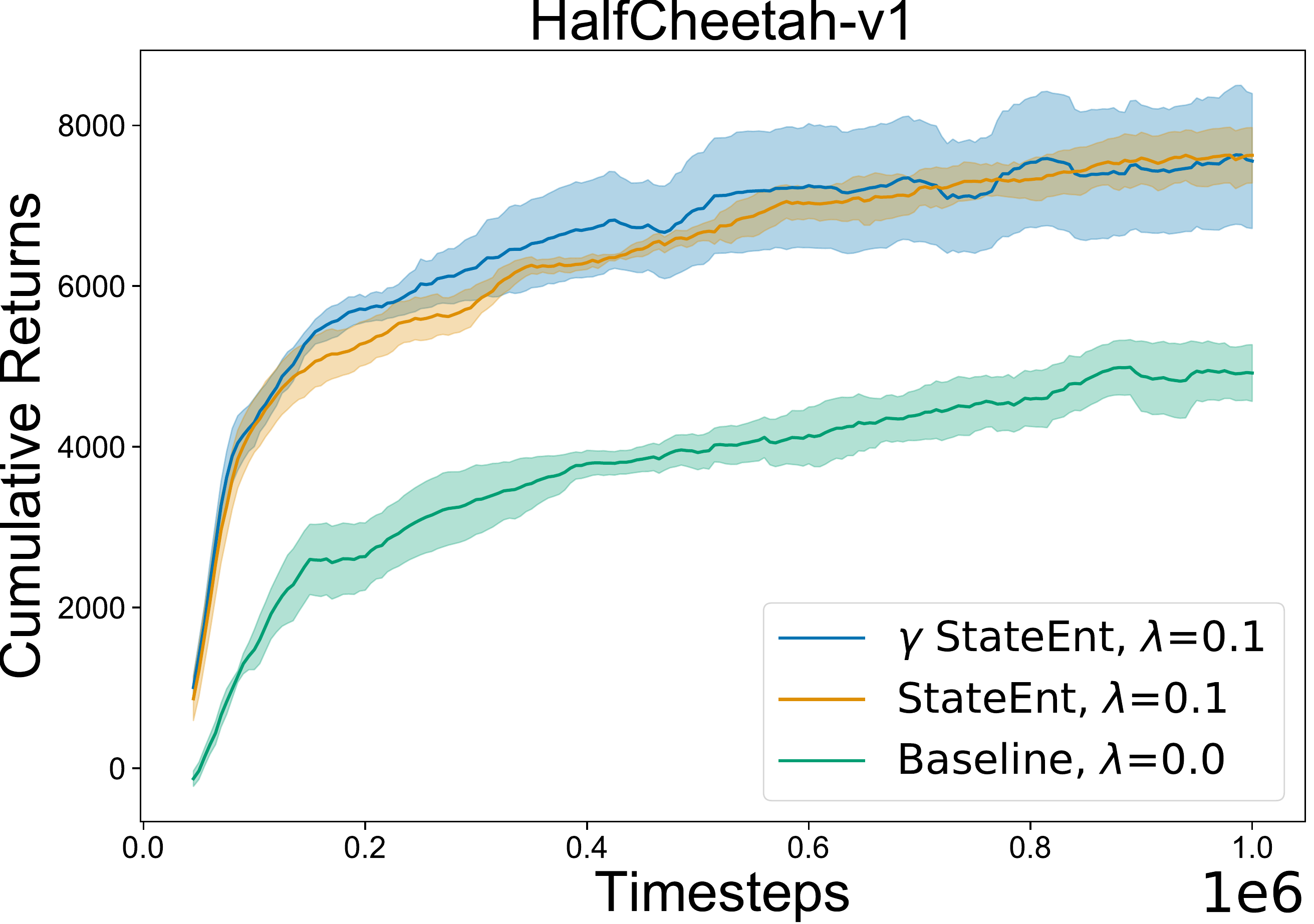}
    \end{subfigure}\hfill
    \begin{subfigure}[b]{0.33\textwidth}
    \centering
        \includegraphics[width=\linewidth]{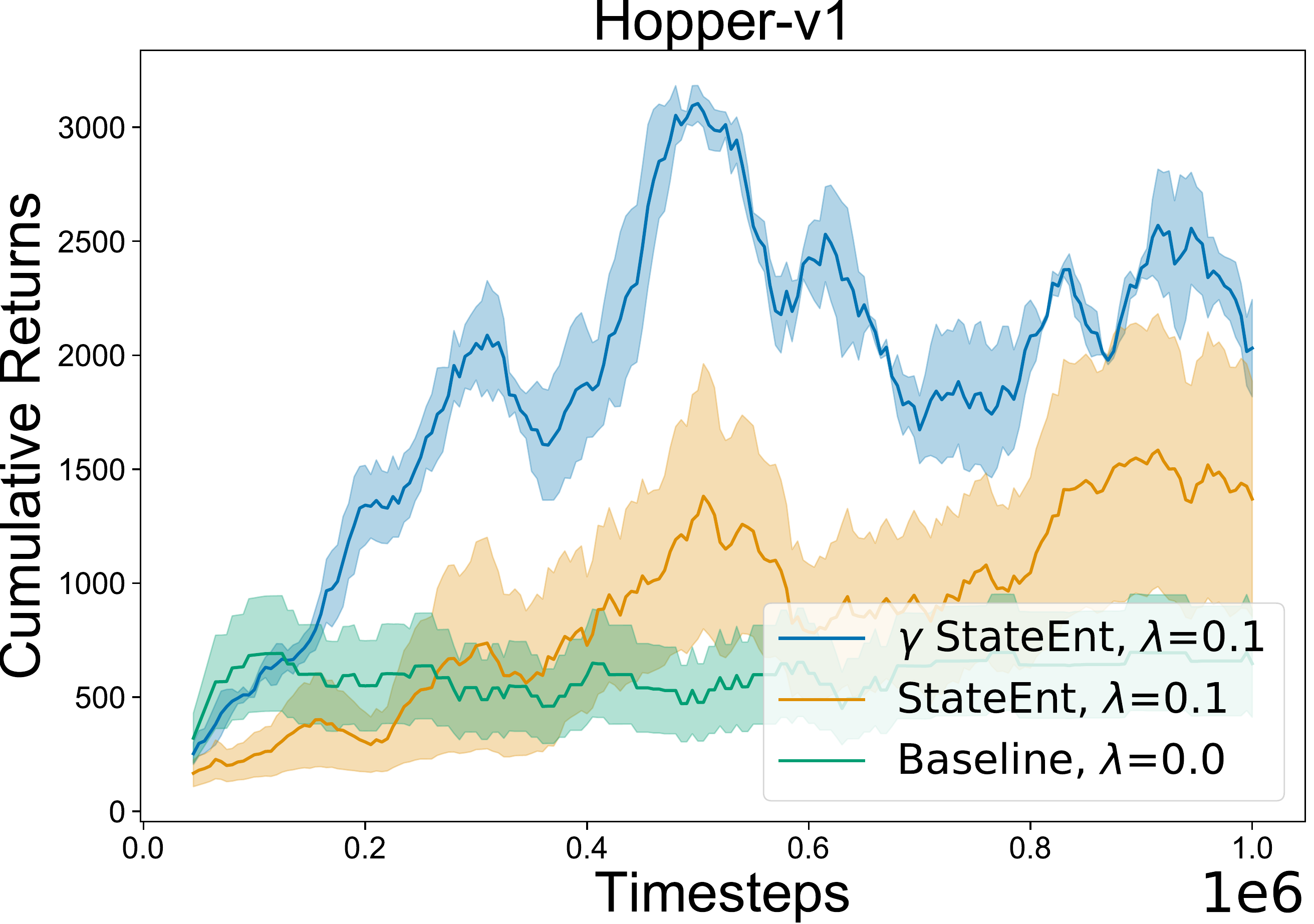}
    \end{subfigure}\hfill
    \begin{subfigure}[b]{0.33\textwidth}
    \centering
        \includegraphics[width=\linewidth]{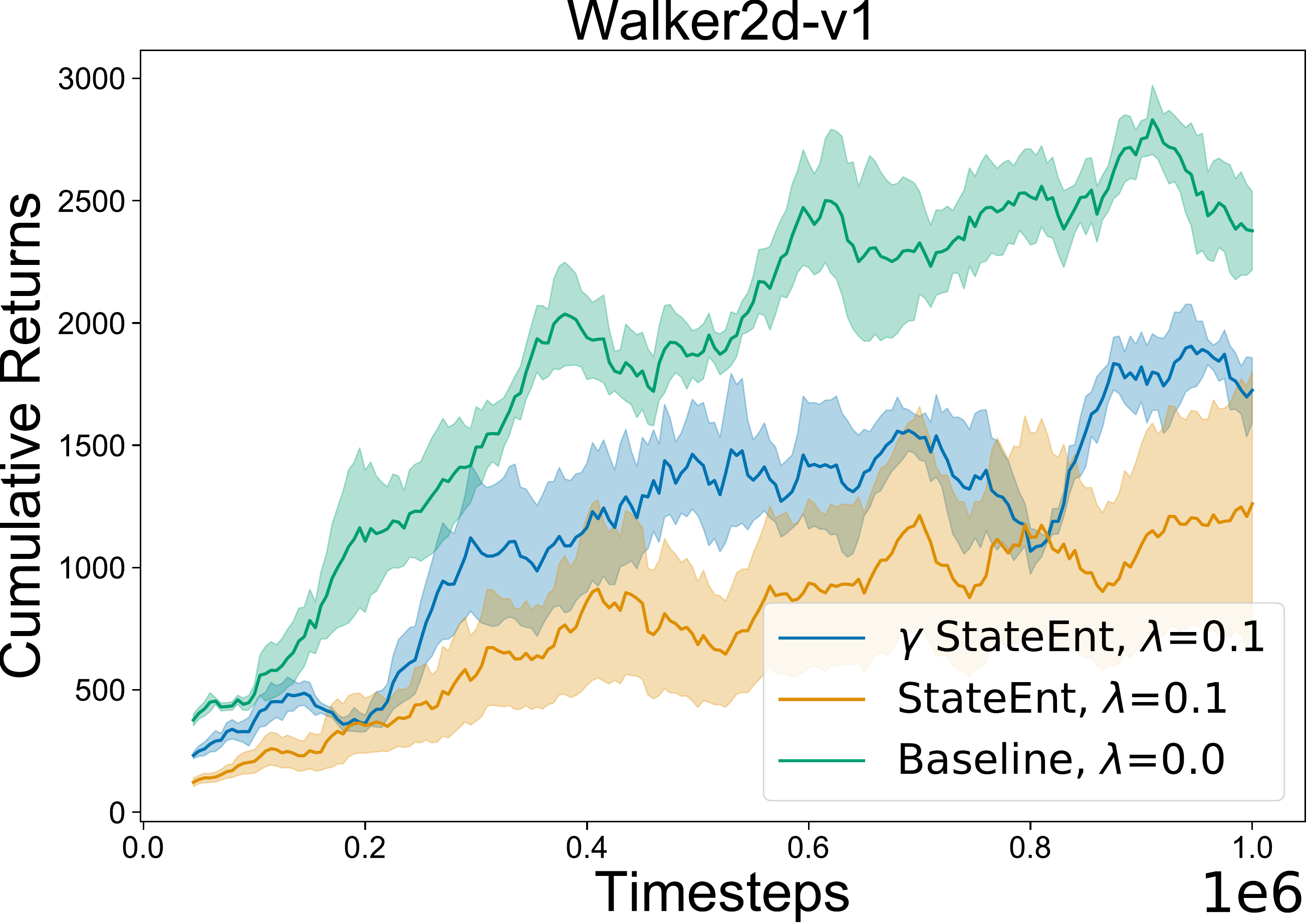}
    \end{subfigure}    
\caption{Significant performance improvements with state entropy regularization with DDPG, especially in tasks like HalfCheetah-v1 and Hopper-v1 where exploration plays a key role. In all the experiments, we use a state entropy regularization coefficient of $\lambda=0.1$ for our approach, and $\lambda=0.0$ for the baseline DDPG.  Experiment results are averaged over $10$ random seeds \citep{DBLP:conf/aaai/0002IBPPM18}. Further experimental results and ablation studies with different $\lambda$ weightings are given in figures~\ref{fig:ddpg_ablation_1},~\ref{fig:ddpg_ablation_2}
in the Appendix}
\label{fig:ddpg_control}
\end{center}    
\end{figure}

We then analyze the significance of state entropy regularization on the soft actor-critic (SAC) framework \citep{sac}. SAC depends on regularizing the policy update with the entropy of the policy. We use the state distribution entropy as an added regularizer to existing maximum policy entropy framework of SAC, and compare performance on the same set of control tasks, as shown in Figure~\ref{fig:sac_control}. 

\begin{figure}[h]
\begin{center}
    \begin{subfigure}[b]{0.33\textwidth}
    \centering
        \includegraphics[width=\linewidth]{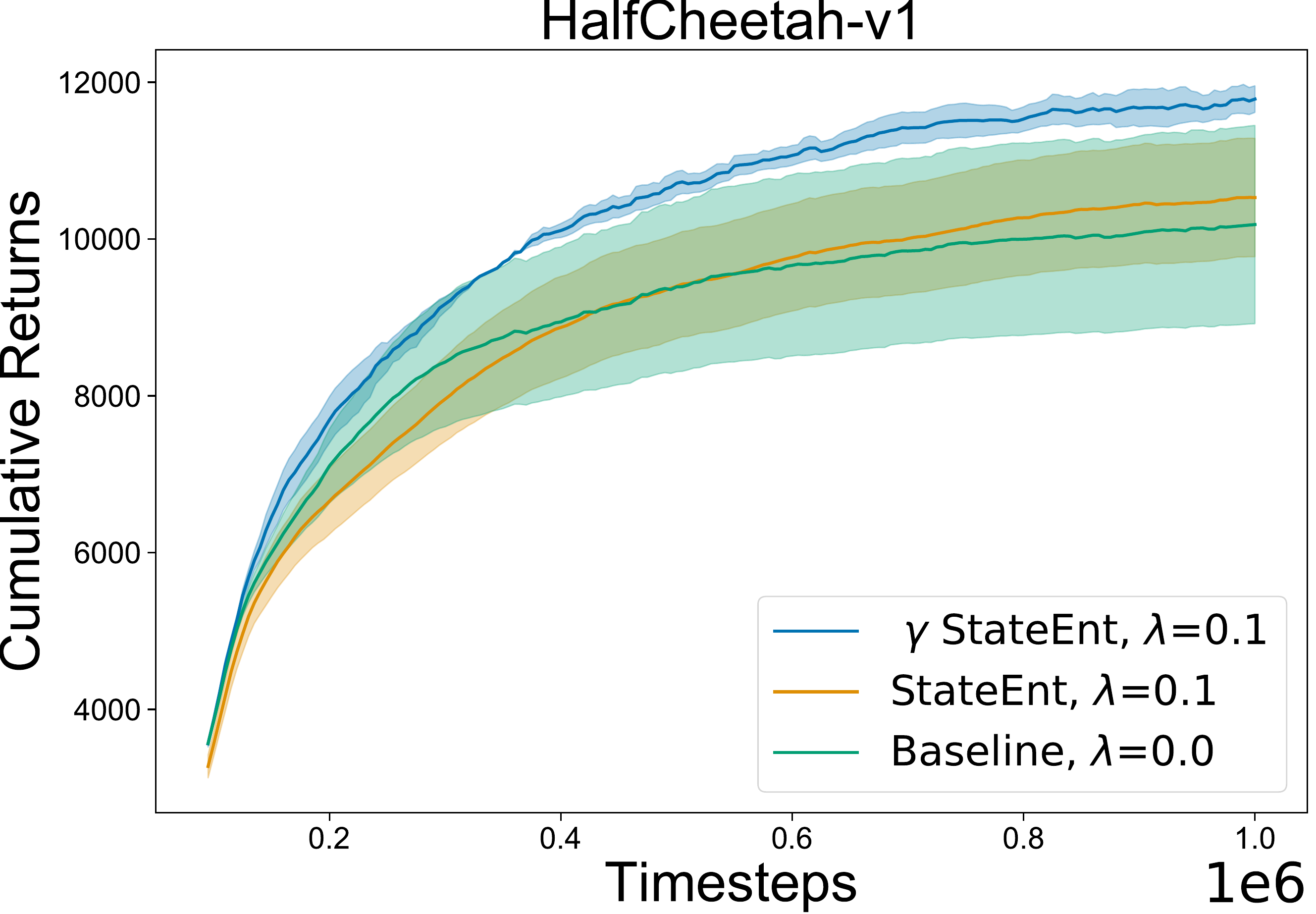}
    \end{subfigure}\hfill
    \begin{subfigure}[b]{0.33\textwidth}
    \centering
        \includegraphics[width=\linewidth]{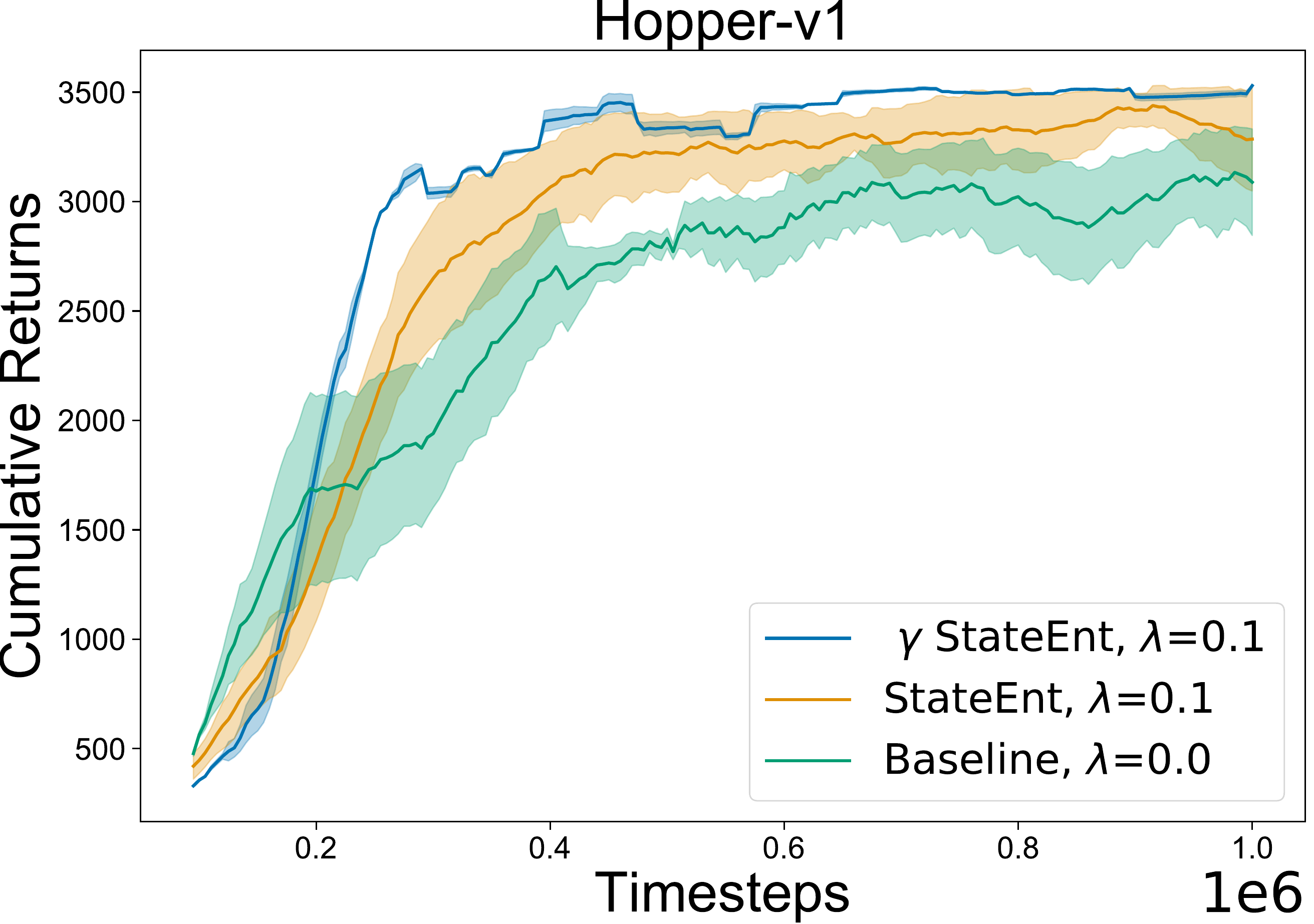}
    \end{subfigure}\hfill
    \begin{subfigure}[b]{0.33\textwidth}
    \centering
        \includegraphics[width=\linewidth]{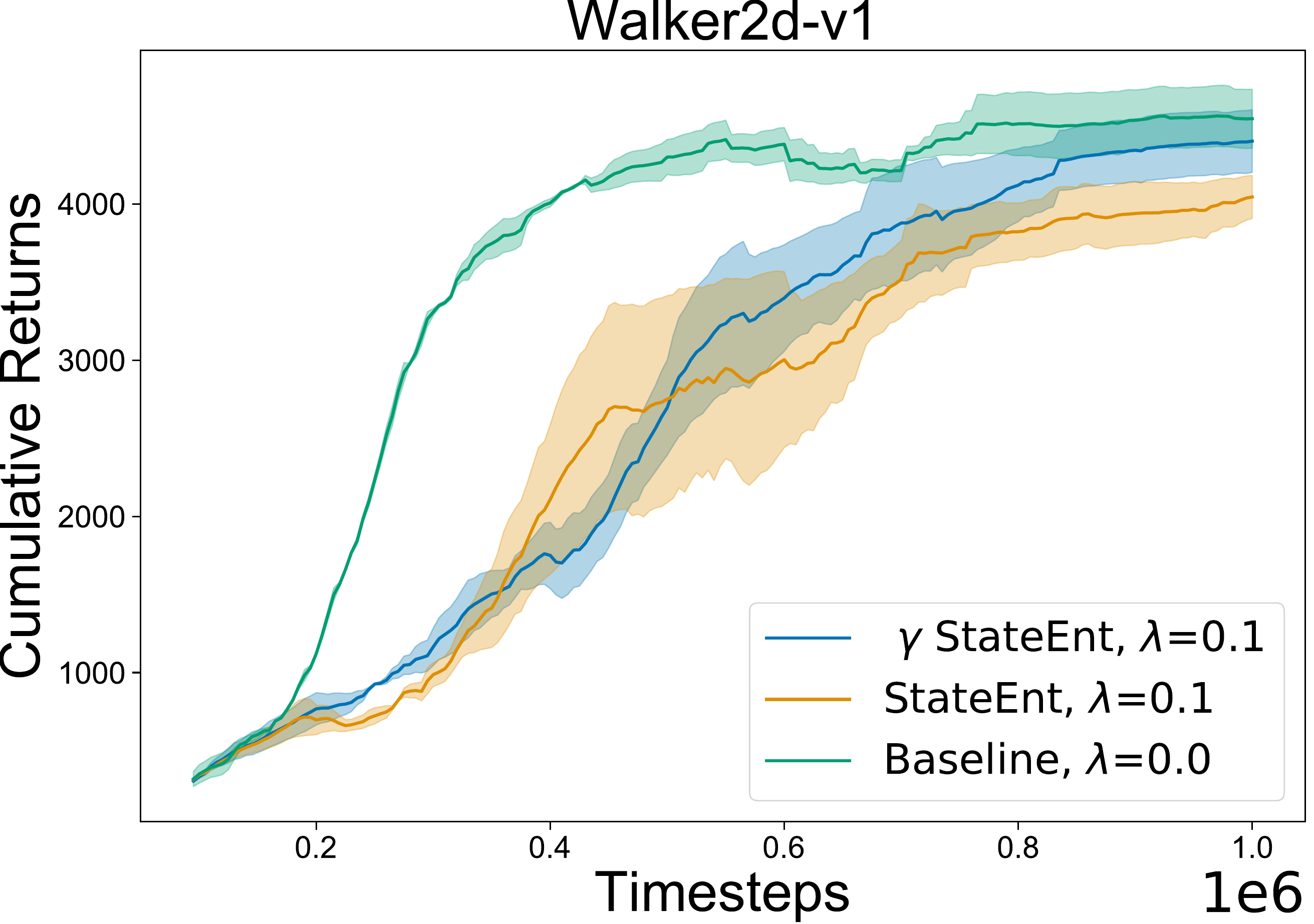}
    \end{subfigure}    
\caption{Significant performance improvements with state entropy regularization with SAC, especially in tasks such as HalfCheetah-v1 where exploration plays a key role. We find $\gamma$ StateEnt for the discounted state distribution entropy regularization to achieve significant performance benefits compared to baseline SAC. Experiment results are averaged over $10$ random seeds \citep{DBLP:conf/aaai/0002IBPPM18}. Further experimental results and ablation studies with $\lambda$ regularization weightings are given in the appendix in figures~\ref{fig:sac_ablation_1},~\ref{fig:sac_ablation_2}. }
\label{fig:sac_control}
\end{center}    
\end{figure}

\section{Related Work}

Existing works in the literature achieve exploration by introducing entropy regularization during policy optimization~\citep{williams1991function,mnih2016asynchronous}. Entropy regularization is commonly used in deep RL tasks \citep{sac, Ziebart} where
preventing policies from quickly collapsing to a deterministic value is the key step for ensuring sufficient exploration. Information theoretic regularizers are also proposed in the literature that help policies to extract useful structures and priors from the tasks. This is often achieved with goal conditioned policies \citep{infobot} or taking information asymmetry under consideration \citep{galashov2019information}. These type of regularization have shown to help with smoothing out the optimization landscape \citep{understanding_entropy}, providing justification of why such methods may work well in practice. Additionally, they induce diversity in the learned policies~\citep{bachman2018vfunc, eysenbach2019diversity} thereby maximizing state space coverage. Other approaches to achieve exploration include reward shaping and adding exploration bonuses \citep{pathak2017curiosity} with the rewards. \citep{Bellmare_Count} introduced a notion of  pseudo count which is derived from sequence of visited states by measuring the number of state occurrences. In~\citep{bellemare_density} this pseudo count is defined in terms of the density model $\rho$ which can be trained on the sequence of states given a fixed policy.

\section{Summary and Discussion}

In this work, we provided a practically feasible algorithm for entropy regularization with the state distributions in policy optimization. We present a practically feasible algorithm, based on estimating the discounted future state distribution, for both episodic and infinite horizon environments. The key to our approach relies on using a density estimator for the state distribution $d_{\pi_{\theta}}$, which is a direct function of the policy parameters $\theta$ itself, such that we can regularize policy optimization to induce policies that can maximize state space coverage. We demonstrate the usefulness of this approach on a wide range of tasks, starting from simple toy tasks to sparse reward gridworld domains, and eventually extending our results to a range of continuous control suites. We re-emphasize that our approach gives a practically convenient handle to deal with the \textit{discounted} state distribution, that are difficult to work with in practice. In addition, we provided a proof of convergence of our method as a three time-scale algorithm, where learning a policy depends on both a value function and a state distribution estimation. 





\bibliography{neurips_2019}
\bibliographystyle{unsrtnat}

\newpage

\section{Appendix : Entropy Regularization with Discounted andStationary State Distribution in Policy Gradient}

\subsection{Convergence of the Three Time-Scale Algorithm}

\begin{corollary}
Let us consider the following iterative updates:
\begin{align}
    \phi_{k+1} &= \phi_{k} - a_k\grad_{\phi}{\mathcal{L}(\phi_k, \theta_k)} \label{eq:phi_update}\\
    \psi_{k+1} &= \psi_k - b_k\grad_{\psi}{TD(\psi_k, \theta_k)}\label{eq:psi_update} \\
    \theta_{k+1} &= \theta_k + c_k \grad_{\theta}{\tilde J(\phi_k, \psi_k, \theta_k)}\label{eq:theta_update}
\end{align},
where $a_k$, $b_k$ and $c_k$ are learning rates.
Furthermore, if we have the following: The parameters $\phi$, $\psi$ and $\theta$ belong to convex and compact metric spaces. The gradients of the density estimator loss function, $\grad_{\phi}{\mathcal{L}(\phi_k, \theta_k)}$, TD loss function $\grad_{\psi}{TD(\psi_k, \theta_k)}$ and the policy gradient $\grad_{\theta}{J(\phi_k, \psi_k, \theta_k)}$ are Lipschitz continuous in their arguments. All these gradient estimators are unbiased and have bounded variance. The ODE corresponding to~\eqref{eq:phi_update} has a unique globally asymptotically stable fixed point, which is Lipschitz continuous with respect to the actor parameters $\theta_k$. The ODE corresponding to~\eqref{eq:psi_update} has a unique globally asymptotically stable fixed point, which is Lipschitz continuous with respect to the actor parameters $\theta_k$. The ODE corresponding to~\eqref{eq:theta_update} has locally asymptotically stable fixed points. The learning rates satisfy:
\[
\sum_{k}a_k = \infty, \sum_{k}b_k = \infty, \sum_{k} c_k = \infty, \sum_{k}a^2_k < \infty, \sum_{k}b^2_k < \infty, \sum_{k}c^2_k < \infty,
\]
\[
\lim_{k \to \infty}\frac{c_k}{a_k} \to 0, \text{ and } \ \lim_{k \to \infty}\frac{c_k}{b_k} \to 0.
\]
then almost surely: The density estimator converges to the stationary state distribution/discounted state distribution corresponding to the actor with parameters $\theta$. The critic converges to the right action-value function for the actor with parameters $\theta$. The actor converges to a policy that is a local maximizer of the modified performance objective, which is a regularized variant of the performance objective.
\end{corollary}
\begin{proof}
The proof is a straightforward extension of the two time-scale algorithm from~\citep{konda2000actor, borkar2009stochastic}. It is to be noted that in our case the density estimator and the critic need to learn at a faster rate than the actor, but since they do not depend on each other, their learning rates can be chosen independent of each other. The difference from the standard two-timescale algorithm then is just the fact that we have two independent parameters to estimate at the faster time scale.
\end{proof}

\subsection{Algorithm Details}

\subsection{Additional Experiment Results : Toy Tasks}
In this section we outline the results obtained two different toy benchmark domains. We used two envrionments from Open AI gym, namely the CartPole-v0 and Acrobot-v1. In both of these environments, we show the performance where we schedule the regularization coefficient $\lambda$ in different ways. We then outline the performance where entropy is maximized by computing both the approximate state distribution and approximate discounted state distribution. For both the cases we followed two different scheduling for the regularization coefficient $\lambda$. In the first case we kept the regularization coefficient $\lambda$ to be constant fixed values and in the second case we slowly decayed $\lambda$ with the episodes. We then made a comparison of these different scheduling with the baselines where we do not use a regularization at all i.e., $\lambda=0$. Our experiments show that using approximate discounted stationary distribution for entropy maximization along with a steady decay of the regularization parameter $\lambda$ with the number of episodes leads to a significant improvement in performances in these domains. 

\begin{figure}[H]
\begin{center}
    \begin{subfigure}[b]{0.33\textwidth}
    \centering
        \includegraphics[width=\linewidth]{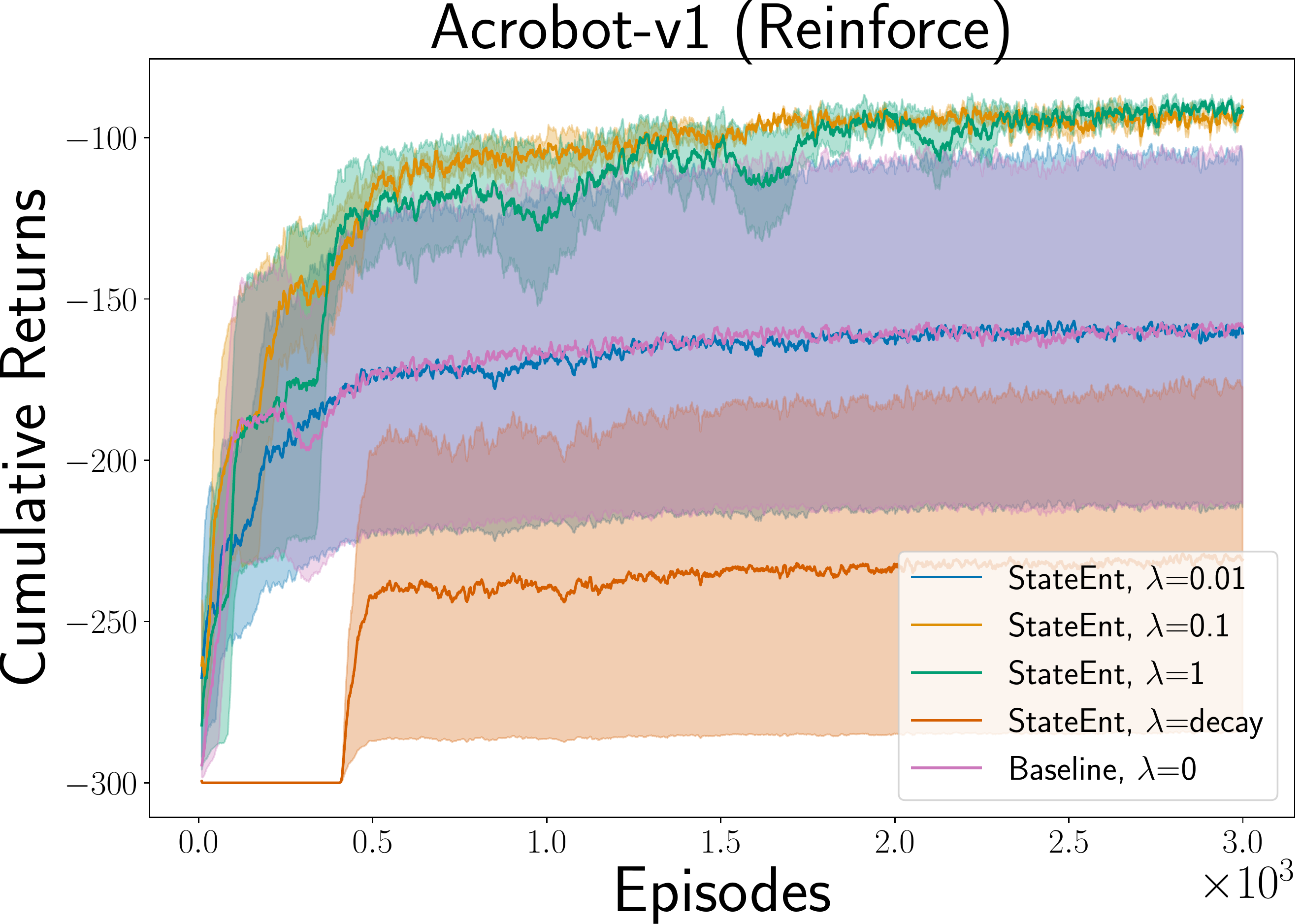}
    \end{subfigure}\hfill
    \begin{subfigure}[b]{0.33\textwidth}
    \centering
        \includegraphics[width=\linewidth]{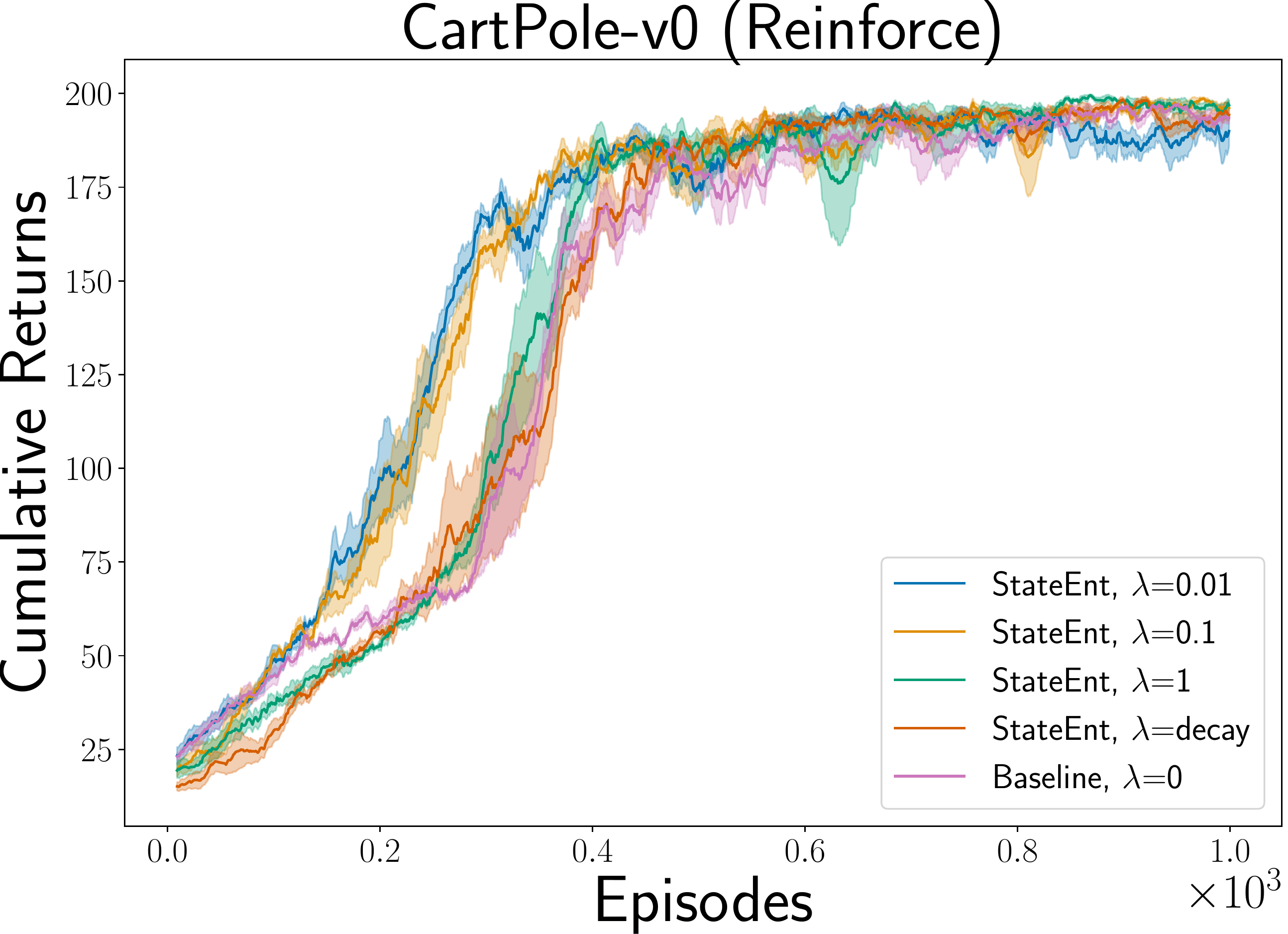}
    \end{subfigure}\hfill
    \begin{subfigure}[b]{0.33\textwidth}
    \centering
        \includegraphics[width=\linewidth]{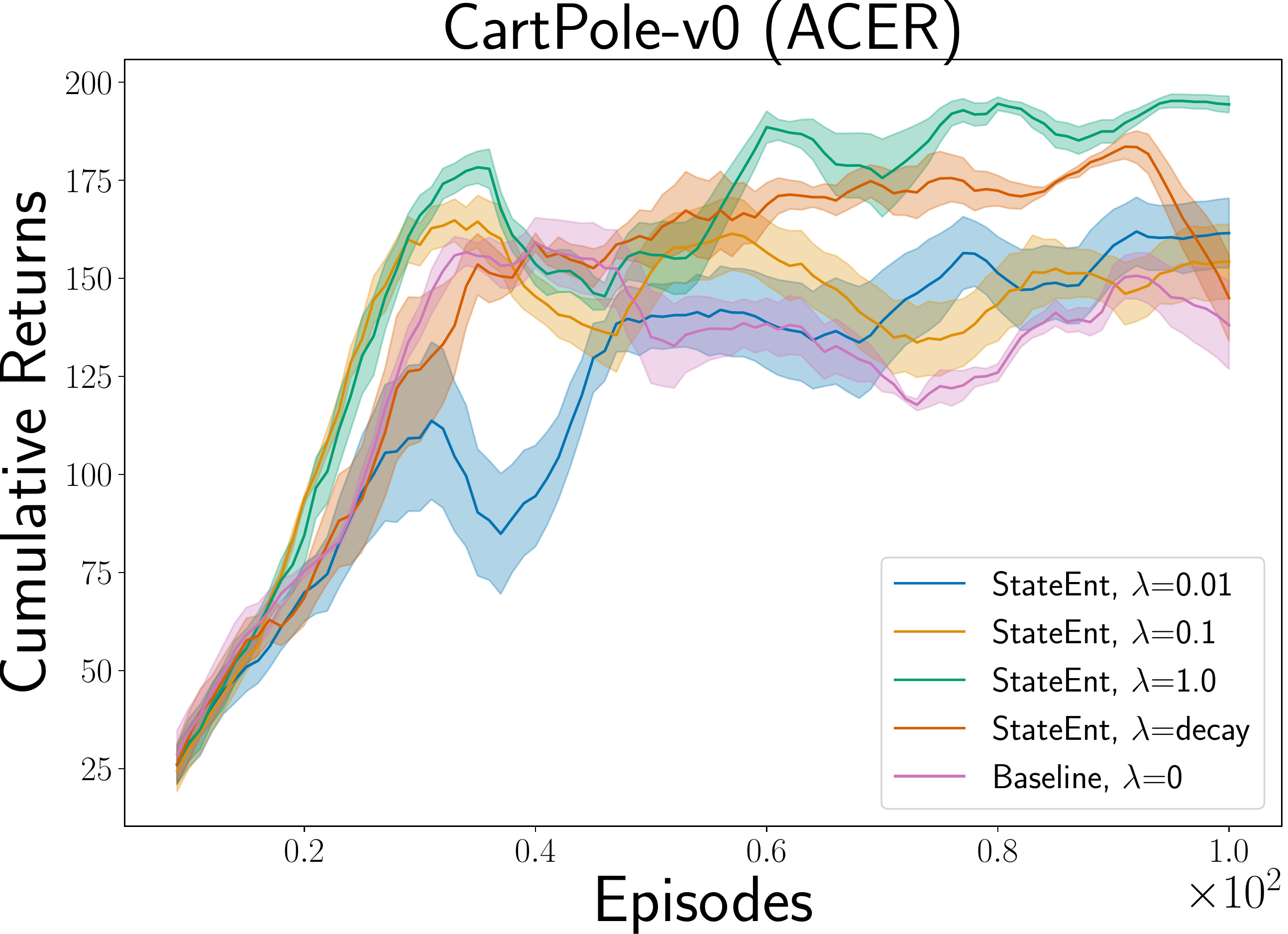}
    \end{subfigure}
    \caption{Performance comparison with difference regularization coefficients. The plots show the value of different regularization coefficients along with the performance obtained when these coefficients are decayed. Experiments where the entropy is computed using \textbf{approximate stationary state distribution} a are demonstrated in two environments.}
    \label{fig:perf_comp_nogamma_cartpole_acrobot}
\end{center}    
\end{figure}

\begin{figure}[h]
\begin{center}
    \begin{subfigure}[b]{0.33\textwidth}
    \centering
        \includegraphics[width=\linewidth]{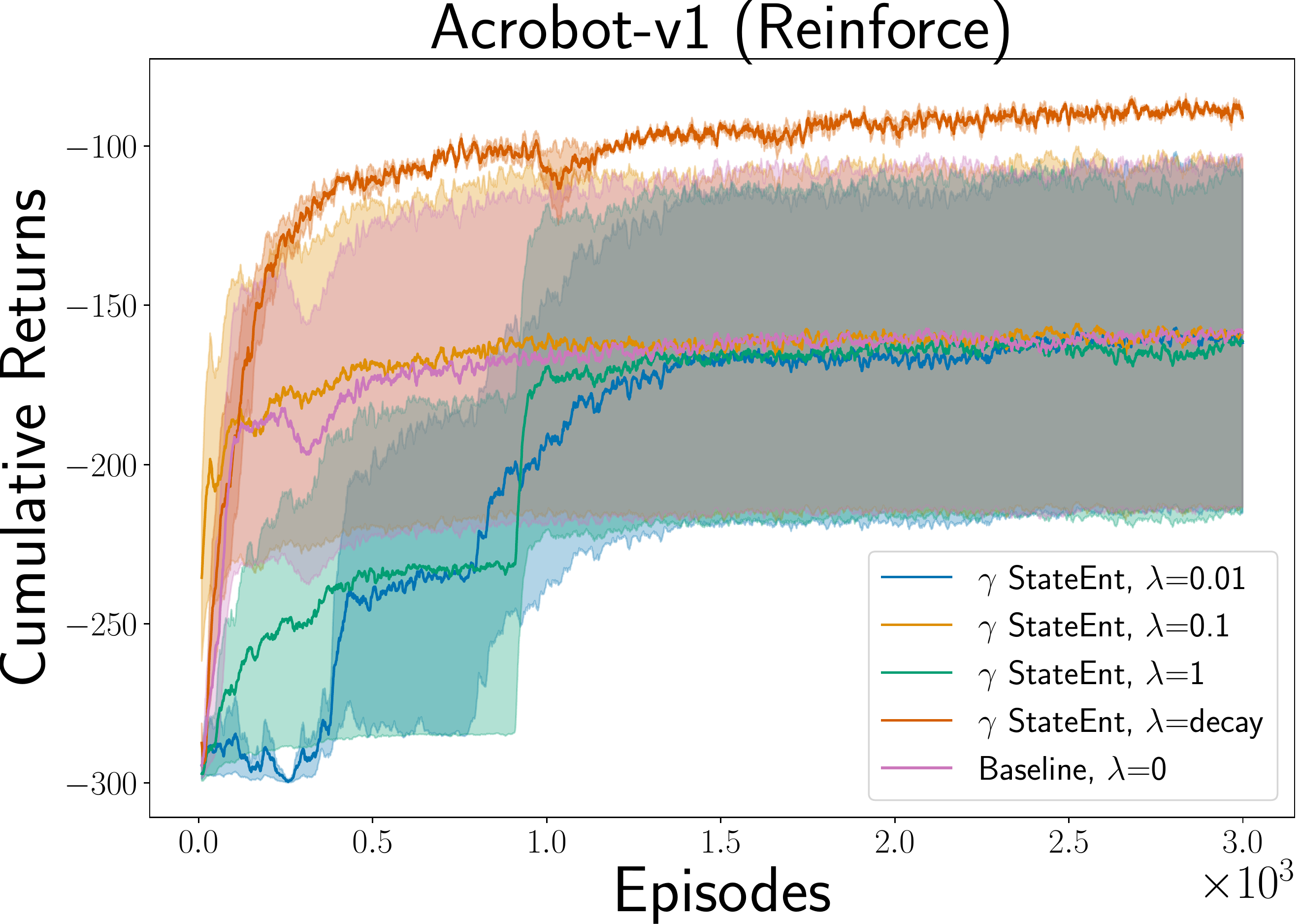}
    \end{subfigure}\hfill
    \begin{subfigure}[b]{0.33\textwidth}
    \centering
        \includegraphics[width=\linewidth]{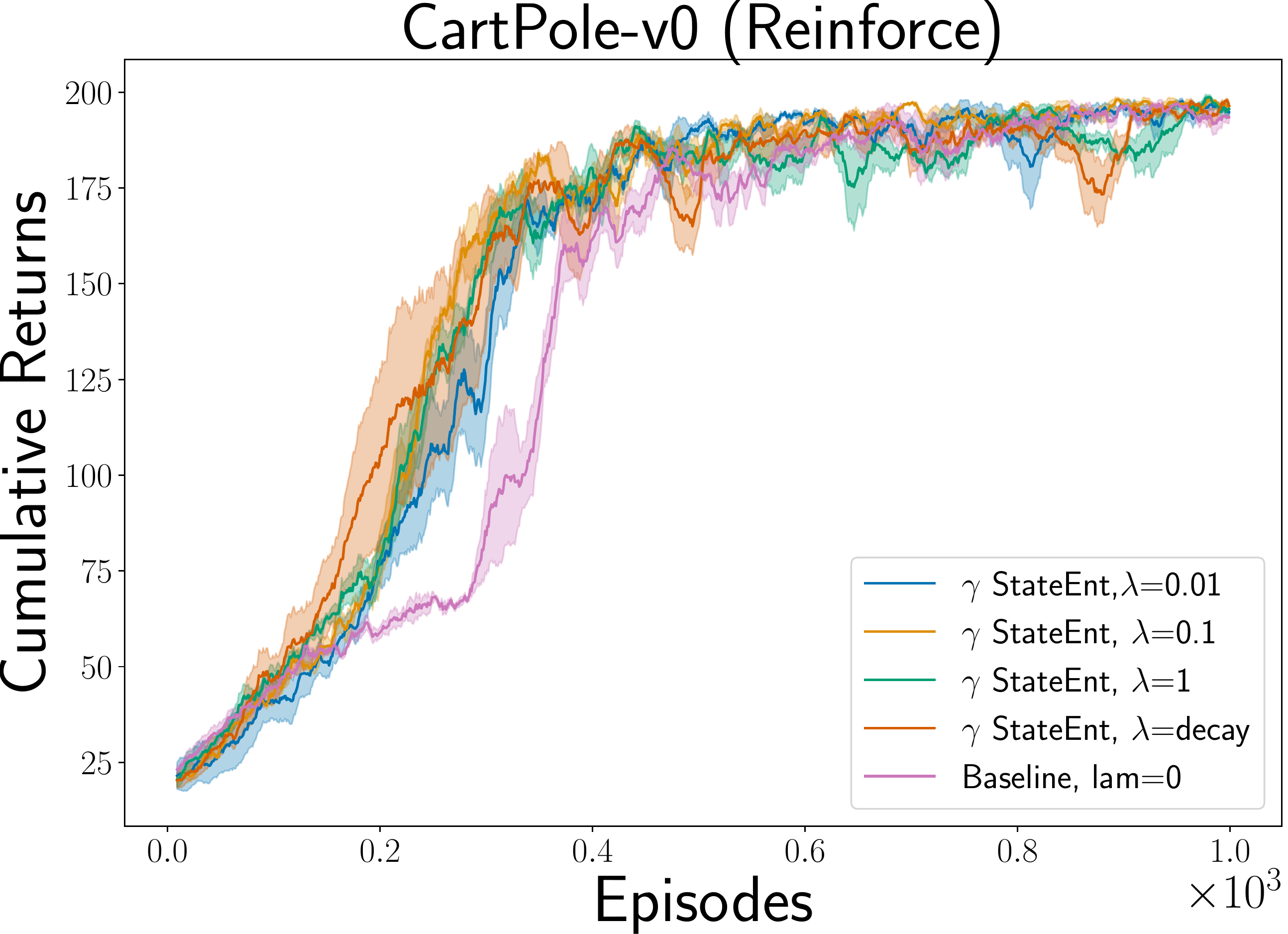}
    \end{subfigure}\hfill
    \begin{subfigure}[b]{0.33\textwidth}
    \centering
        \includegraphics[width=\linewidth]{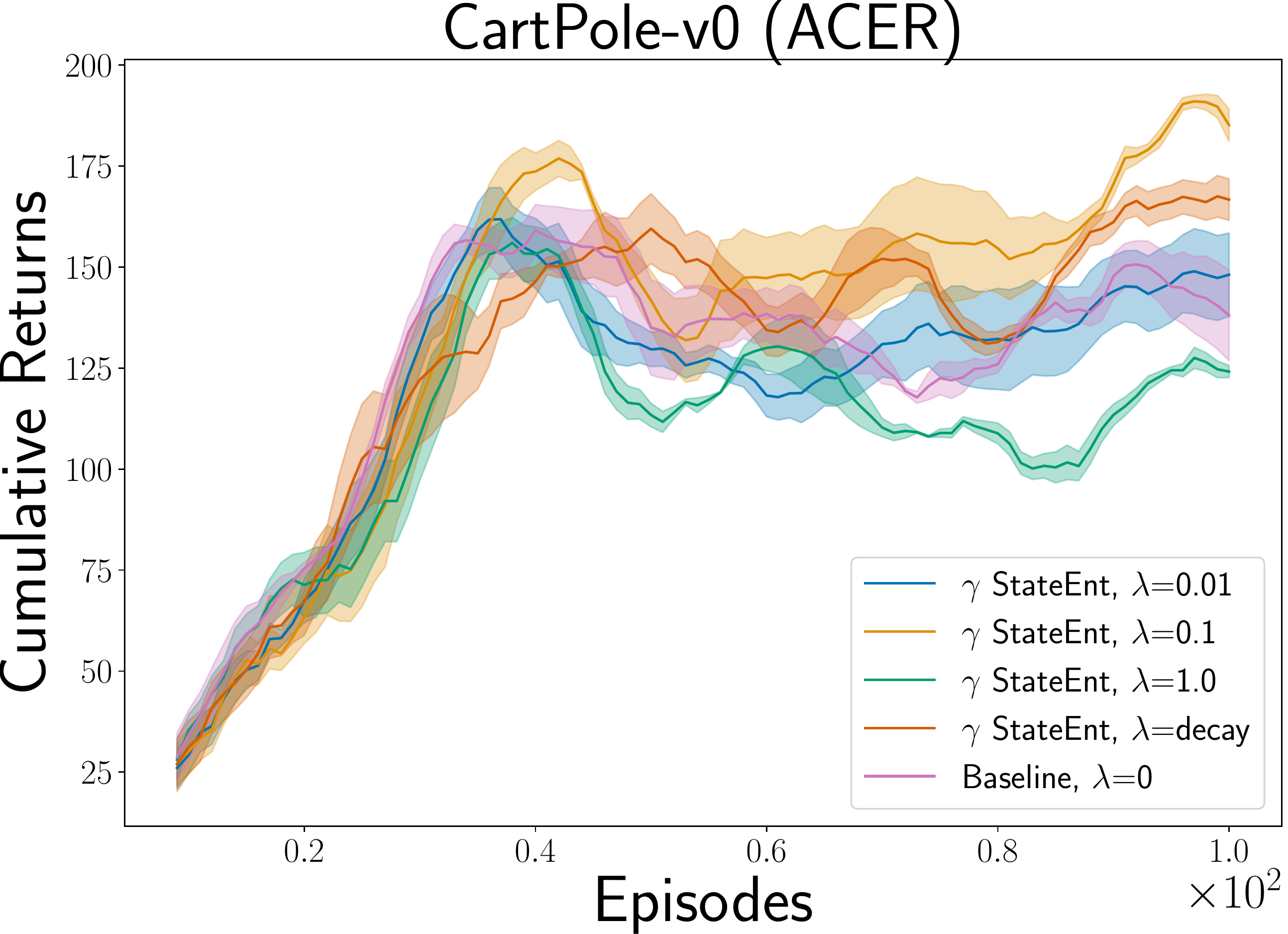}
    \end{subfigure}
    \caption{Performance comparison with difference regularization coefficients. The plots show the value of different regularization coefficients along with the performance obtained when these coefficients are decayed. Experiments where the entropy is computed using \textbf{approximate discounted state distribution} are demonstrated in two environments.}
    \label{fig:perf_comp_gamma_cartpole_acrobot}
\end{center}    
\end{figure}

\section{Experimental Results: Continuous Control Tasks}
We have further included our ablation studies for different continuous control tasks. We have primarily used two popular policy gradient algorithms DDPG and SAC. For the environments, we demonstrated our results in three different mujoco domains which are HalfCheetah-v1, Hopper-v1 and Walker2d-v1. Out plot outlines different performances obtained when the entropy is maximized either by computing the stationary state distribution or the discounted state distribution. We show that maximizing entropy with either the stationary state distribution or the discounted state distribution significantly leads to a better performance in these domains. 
\subsection{Additional Experiment Results : DDPG}

\begin{figure}[H]
\begin{center}
    \begin{subfigure}[b]{0.33\textwidth}
    \centering
        \includegraphics[width=\linewidth]{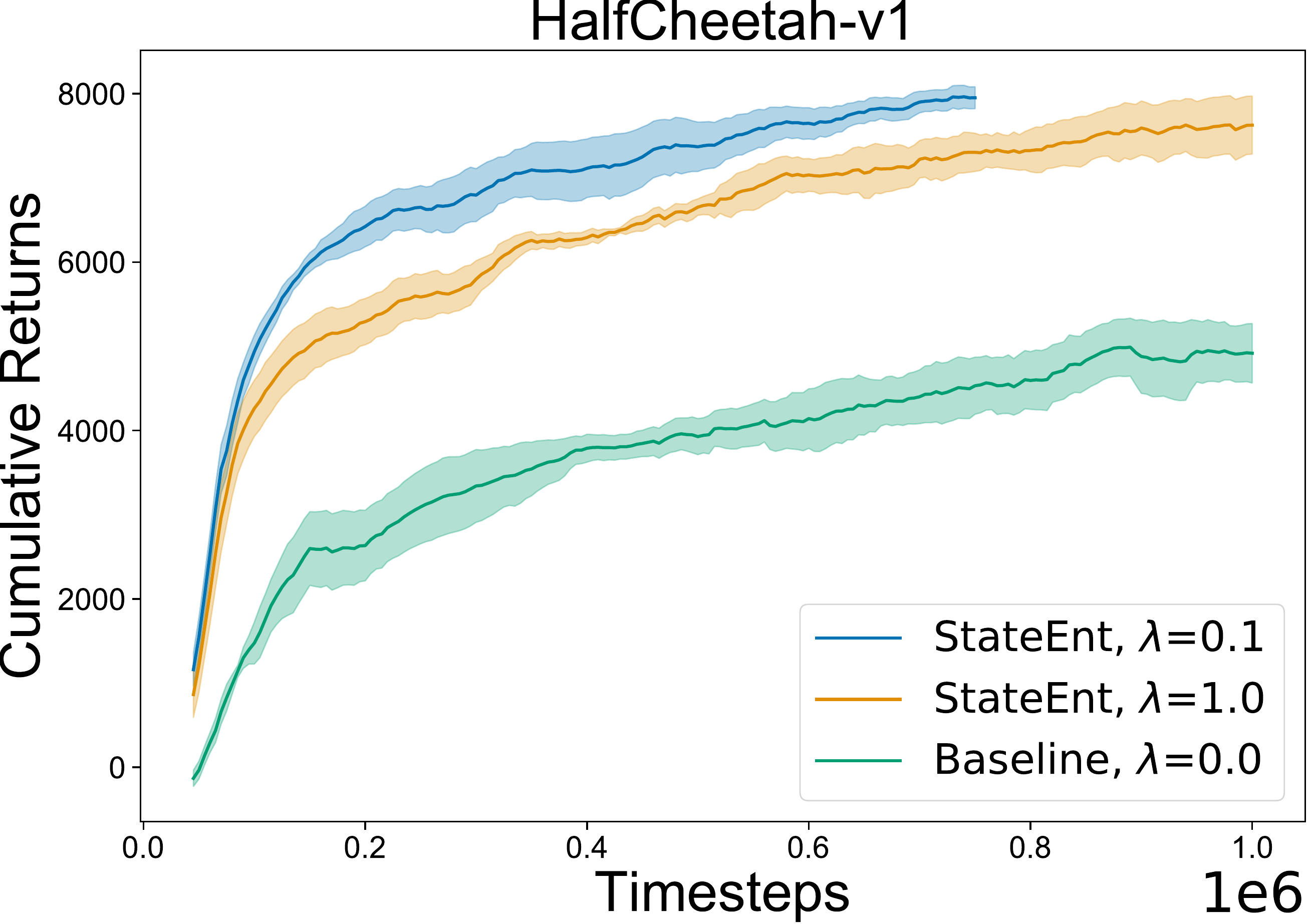}
    \end{subfigure}\hfill
    \begin{subfigure}[b]{0.33\textwidth}
    \centering
        \includegraphics[width=\linewidth]{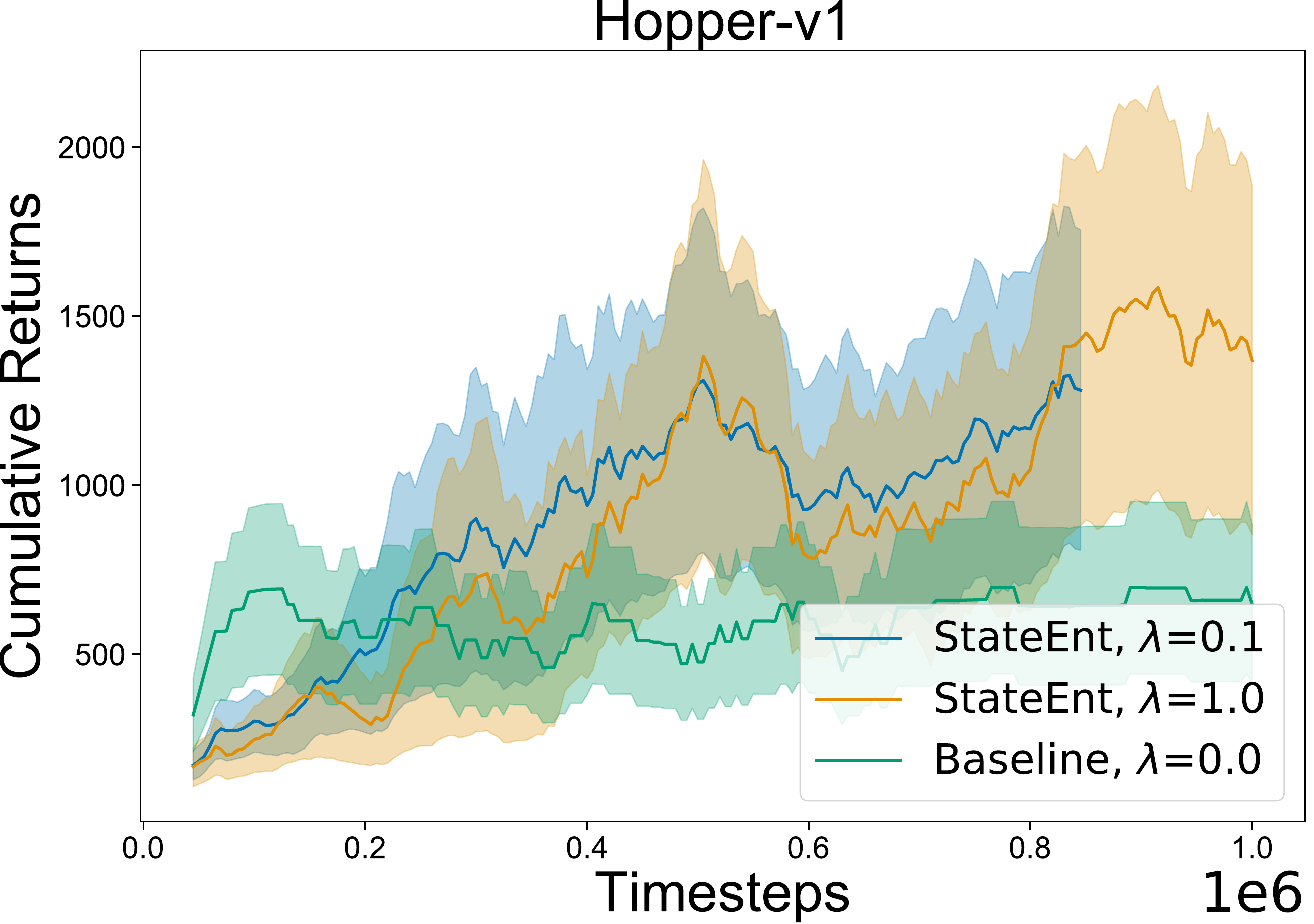}
    \end{subfigure}\hfill
    \begin{subfigure}[b]{0.33\textwidth}
    \centering
        \includegraphics[width=\linewidth]{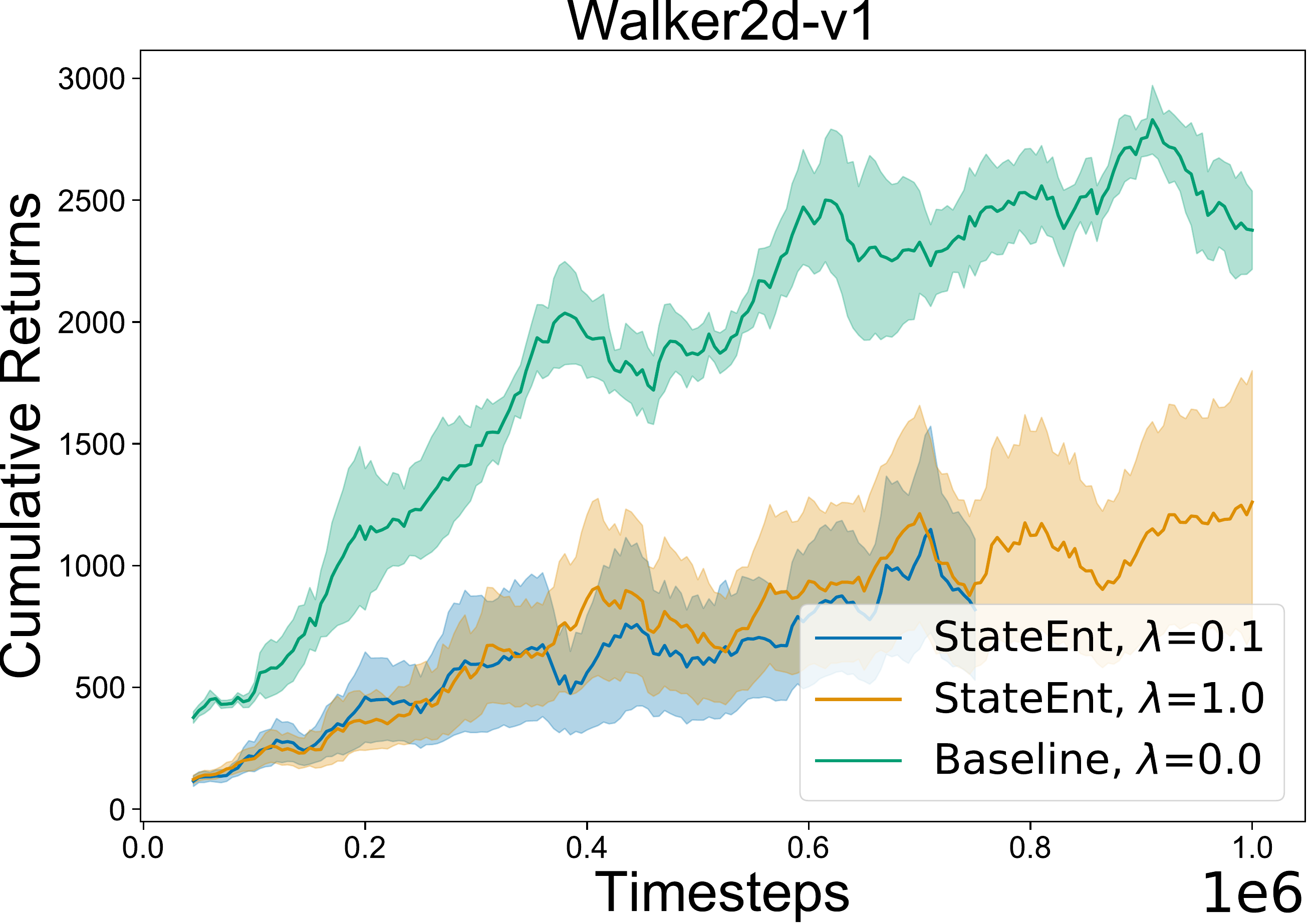}
    \end{subfigure}    
\caption{Additional experimental results with DDPG with StateEnt, for entropy regularization with stationary state distribution. We analyse the significance of using the stationary state distribution over control tasks, with different ranges of $\lambda$ parameters. Note that, for our experiments, we did not do an extensive hyperparameter tuning, but only with the range of $\lambda$ hyperparameters presented in the ablation study here.}
\label{fig:ddpg_ablation_1}
\end{center}    
\end{figure}

\begin{figure}[H]
\begin{center}
    \begin{subfigure}[b]{0.33\textwidth}
    \centering
        \includegraphics[width=\linewidth]{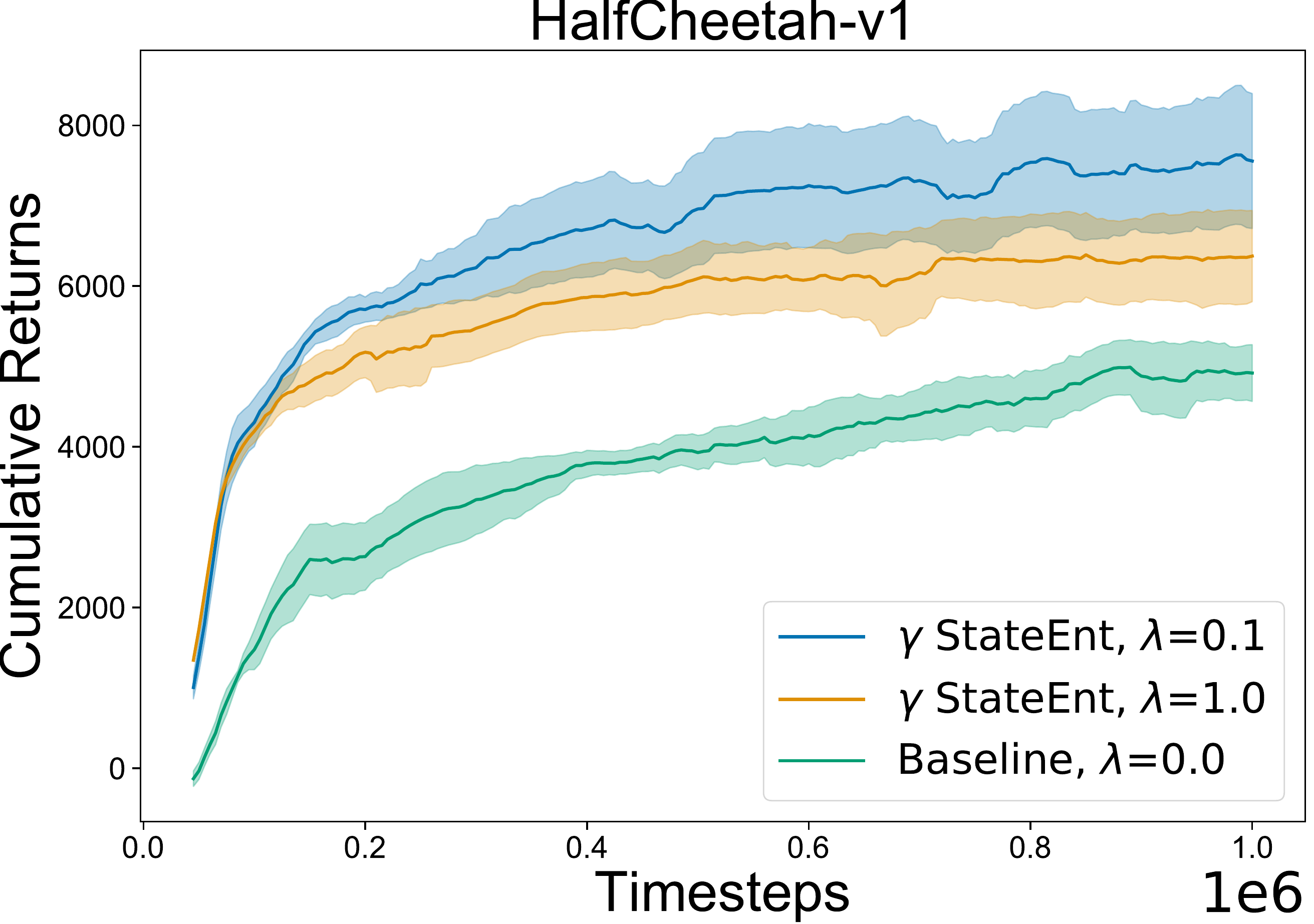}
    \end{subfigure}\hfill
    \begin{subfigure}[b]{0.33\textwidth}
    \centering
        \includegraphics[width=\linewidth]{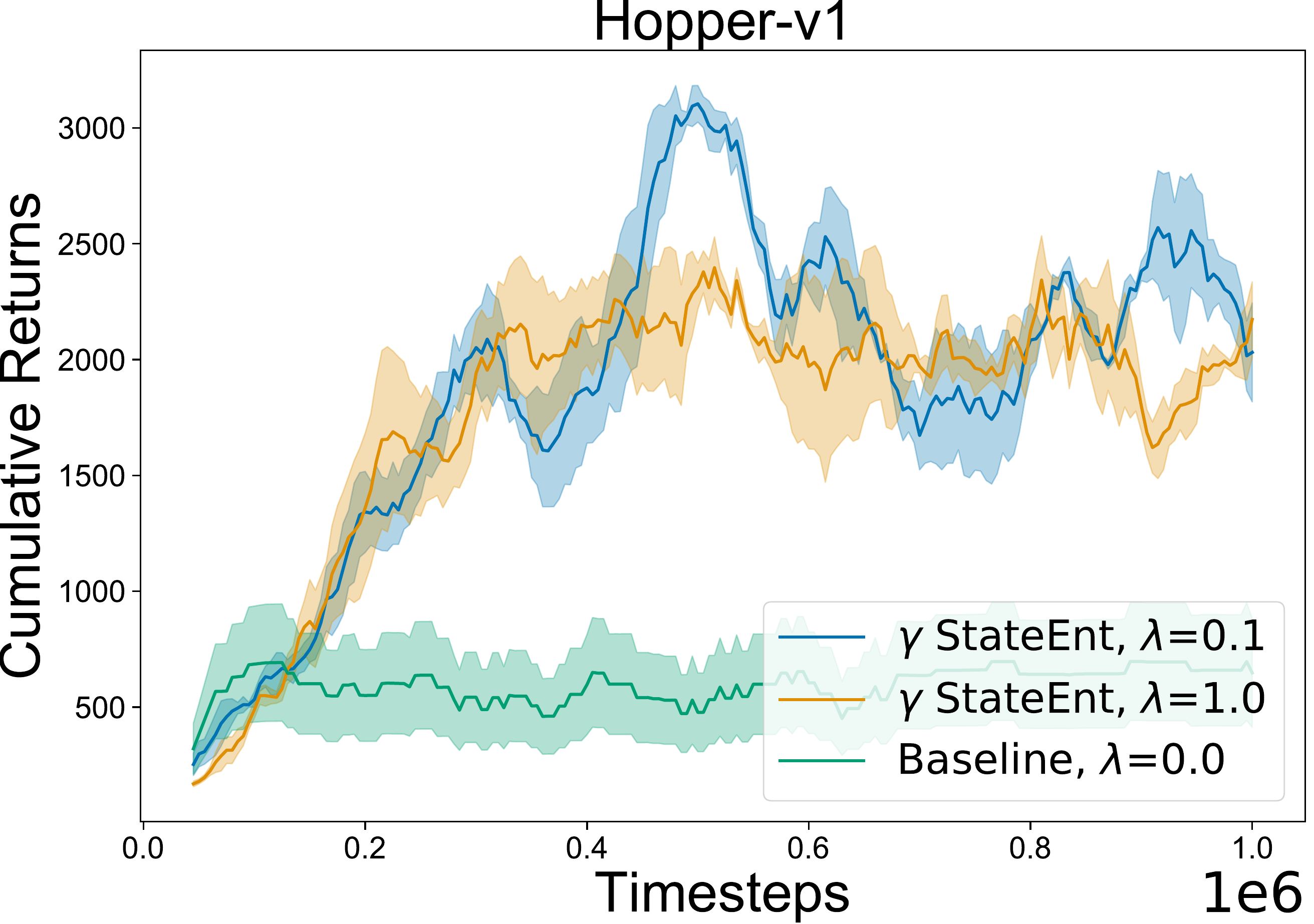}
    \end{subfigure}\hfill
    \begin{subfigure}[b]{0.33\textwidth}
    \centering
        \includegraphics[width=\linewidth]{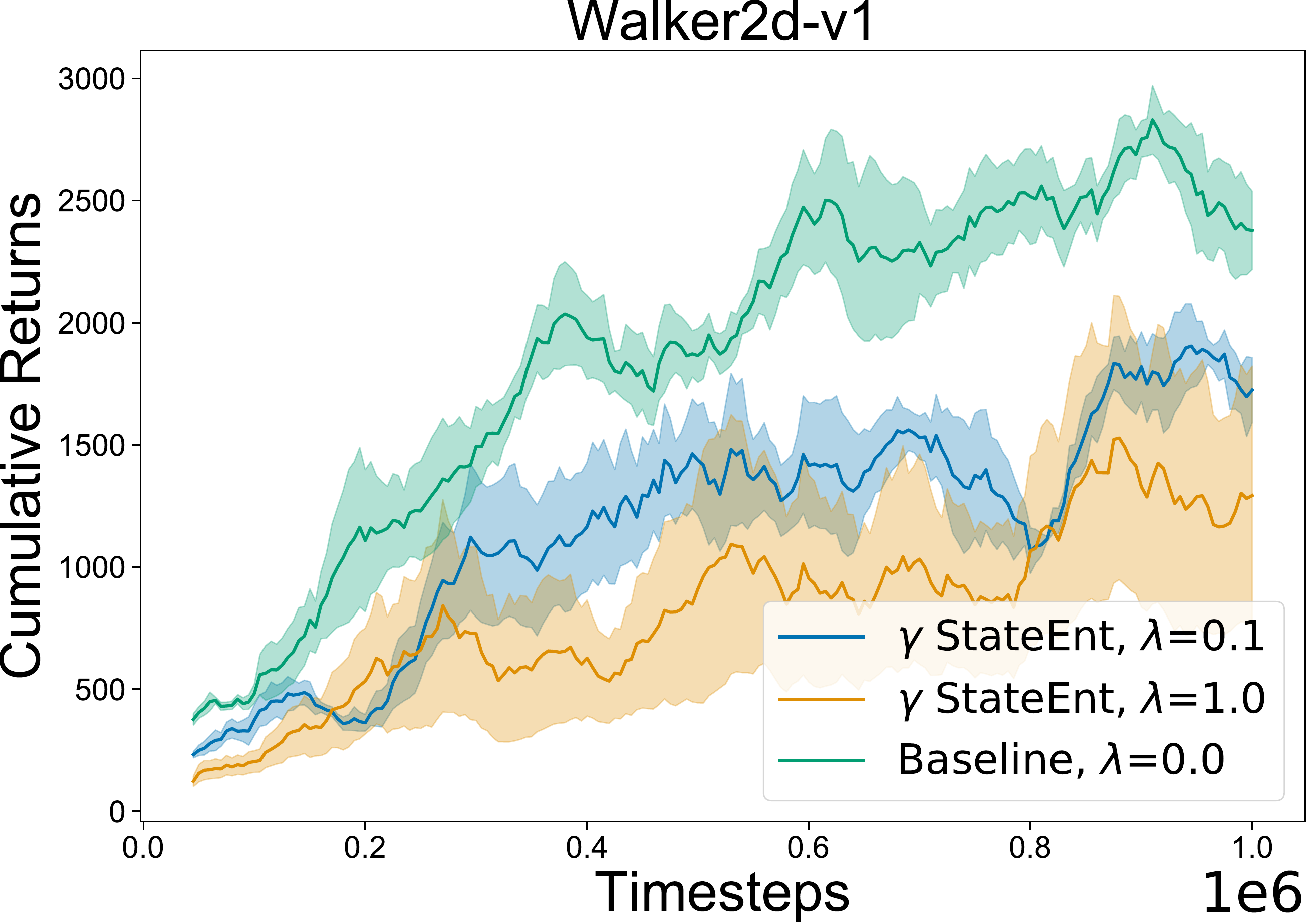}
    \end{subfigure}    
\caption{Additional experimental results with DDPG with $\gamma$ StateEnt, for entropy regularization with discounted state distribution. We analyse the significance of using the discounted state distribution over control tasks, with different ranges of $\lambda$ parameters. Note that, for our experiments, we did not do an extensive hyperparameter tuning, but only with the range of $\lambda$ hyperparameters presented in the ablation study here.}
\label{fig:ddpg_ablation_2}
\end{center}    
\end{figure}

\subsection{Additional Experiment Results : SAC}

\begin{figure}[H]
\begin{center}
    \begin{subfigure}[b]{0.33\textwidth}
    \centering
        \includegraphics[width=\linewidth]{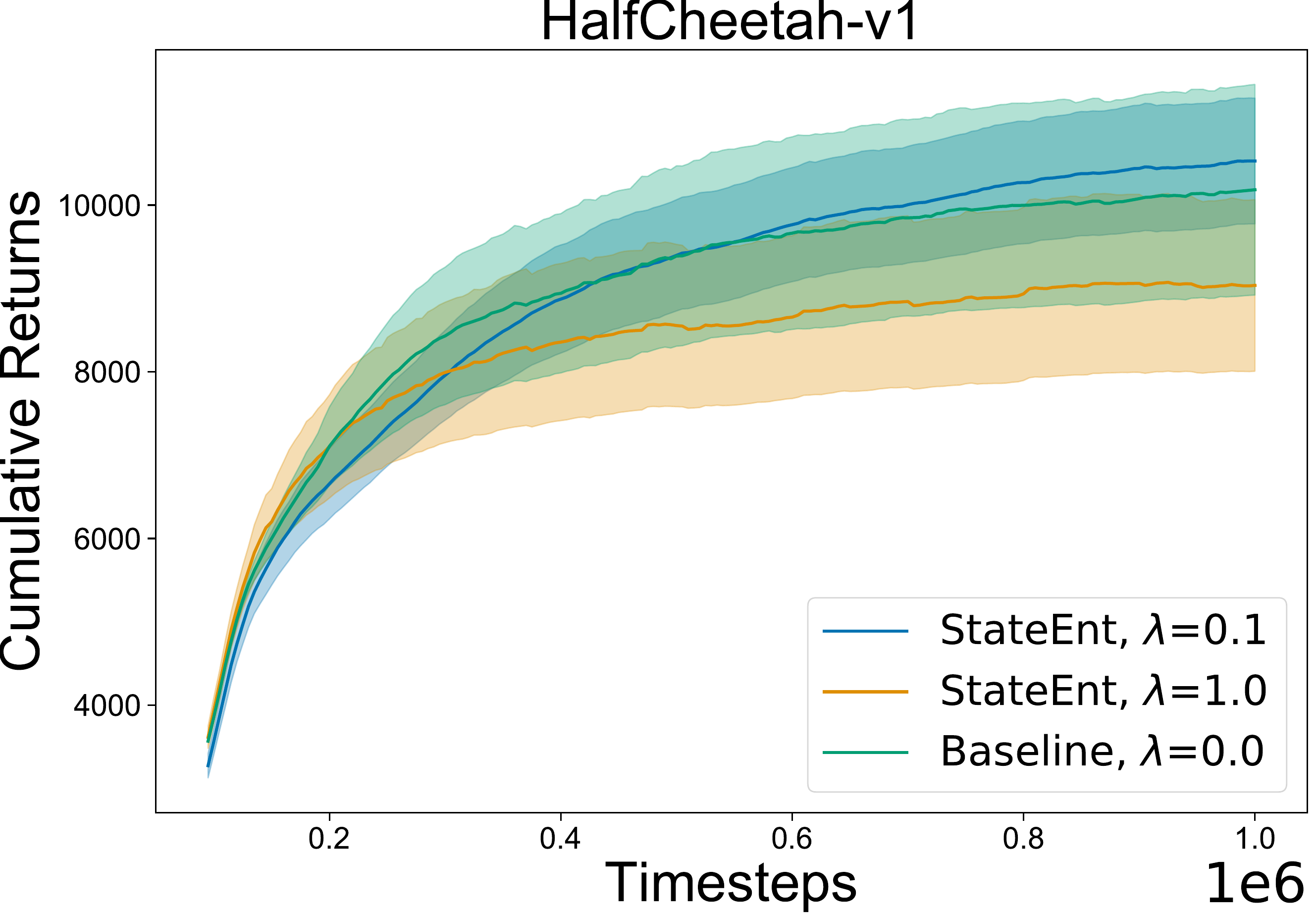}
    \end{subfigure}\hfill
    \begin{subfigure}[b]{0.33\textwidth}
    \centering
        \includegraphics[width=\linewidth]{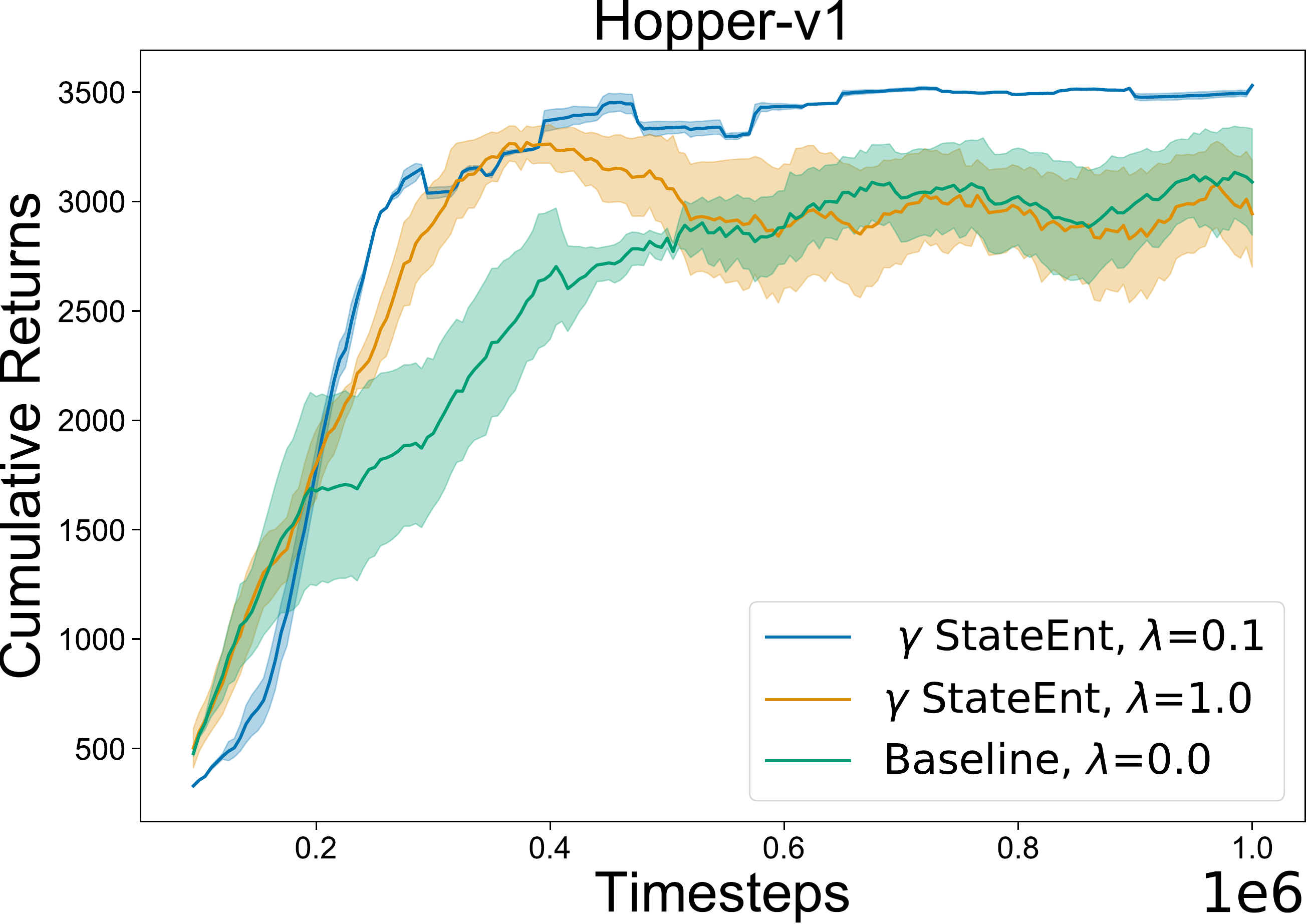}
    \end{subfigure}\hfill
    \begin{subfigure}[b]{0.33\textwidth}
    \centering
        \includegraphics[width=\linewidth]{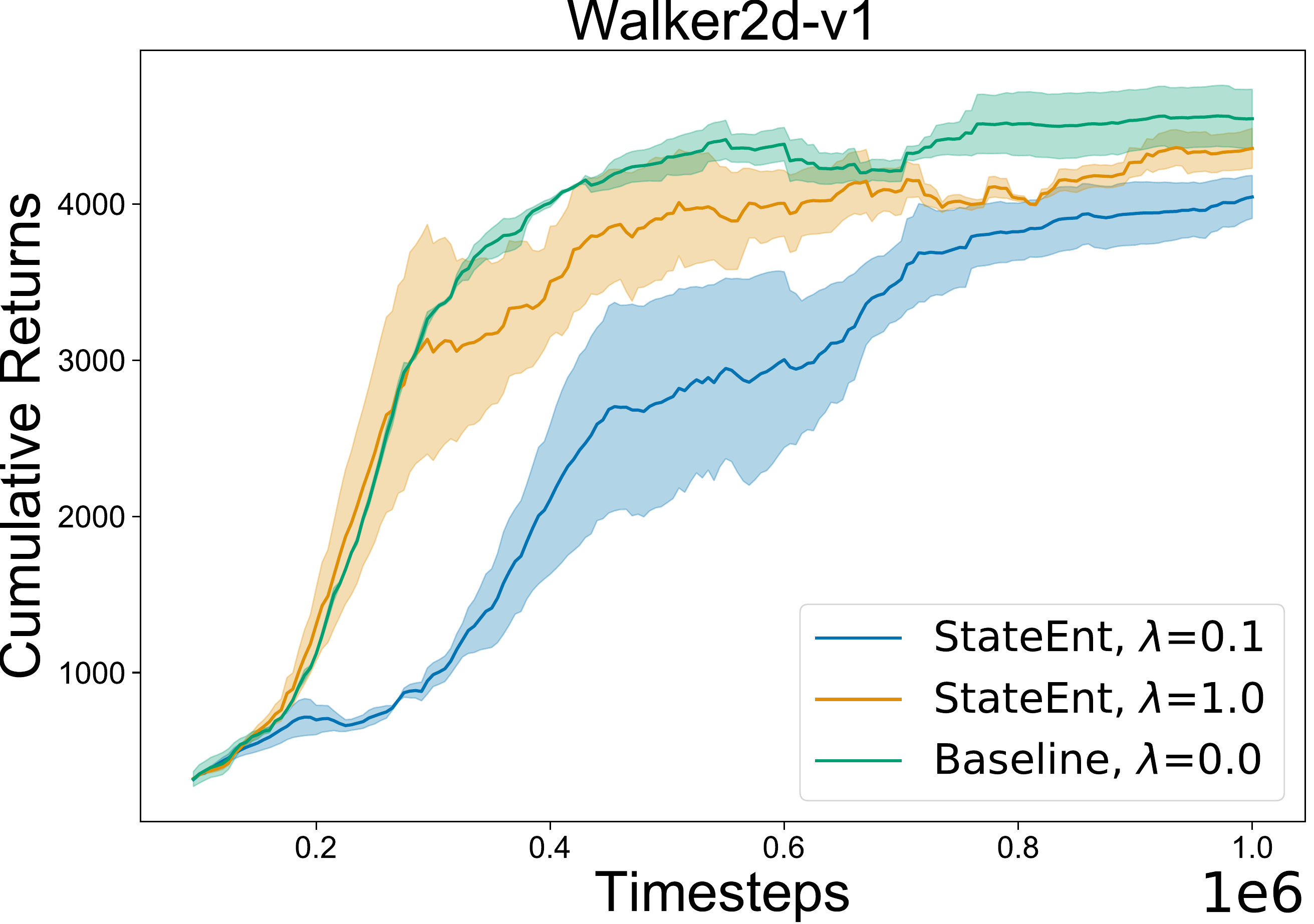}
    \end{subfigure}    
\caption{Additional experimental results with SAC with StateEnt, for entropy regularization with stationary state distribution. We analyse the significance of using the stationary state distribution over control tasks, with different ranges of $\lambda$ parameters. Note that, for our experiments, we did not do an extensive hyperparameter tuning, but only with the range of $\lambda$ hyperparameters presented in the ablation study here.}
\label{fig:sac_ablation_1}
\end{center}    
\end{figure}
\begin{figure}[H]
\begin{center}
    \begin{subfigure}[b]{0.33\textwidth}
    \centering
        \includegraphics[width=\linewidth]{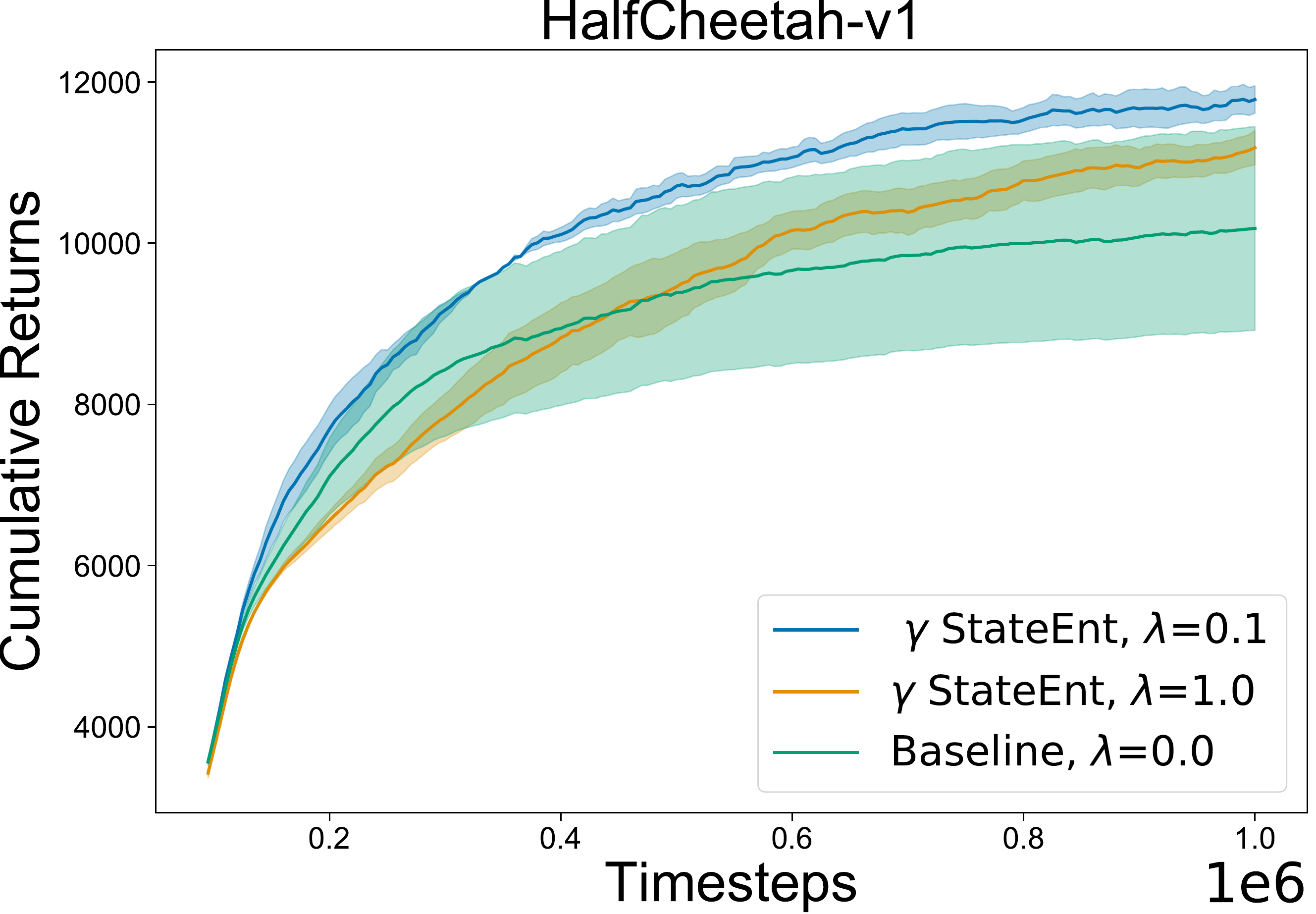}
    \end{subfigure}\hfill
    \begin{subfigure}[b]{0.33\textwidth}
    \centering
        \includegraphics[width=\linewidth]{SAC_with_gamma_Hopper-v1.pdf}
    \end{subfigure}\hfill
    \begin{subfigure}[b]{0.33\textwidth}
    \centering
        \includegraphics[width=\linewidth]{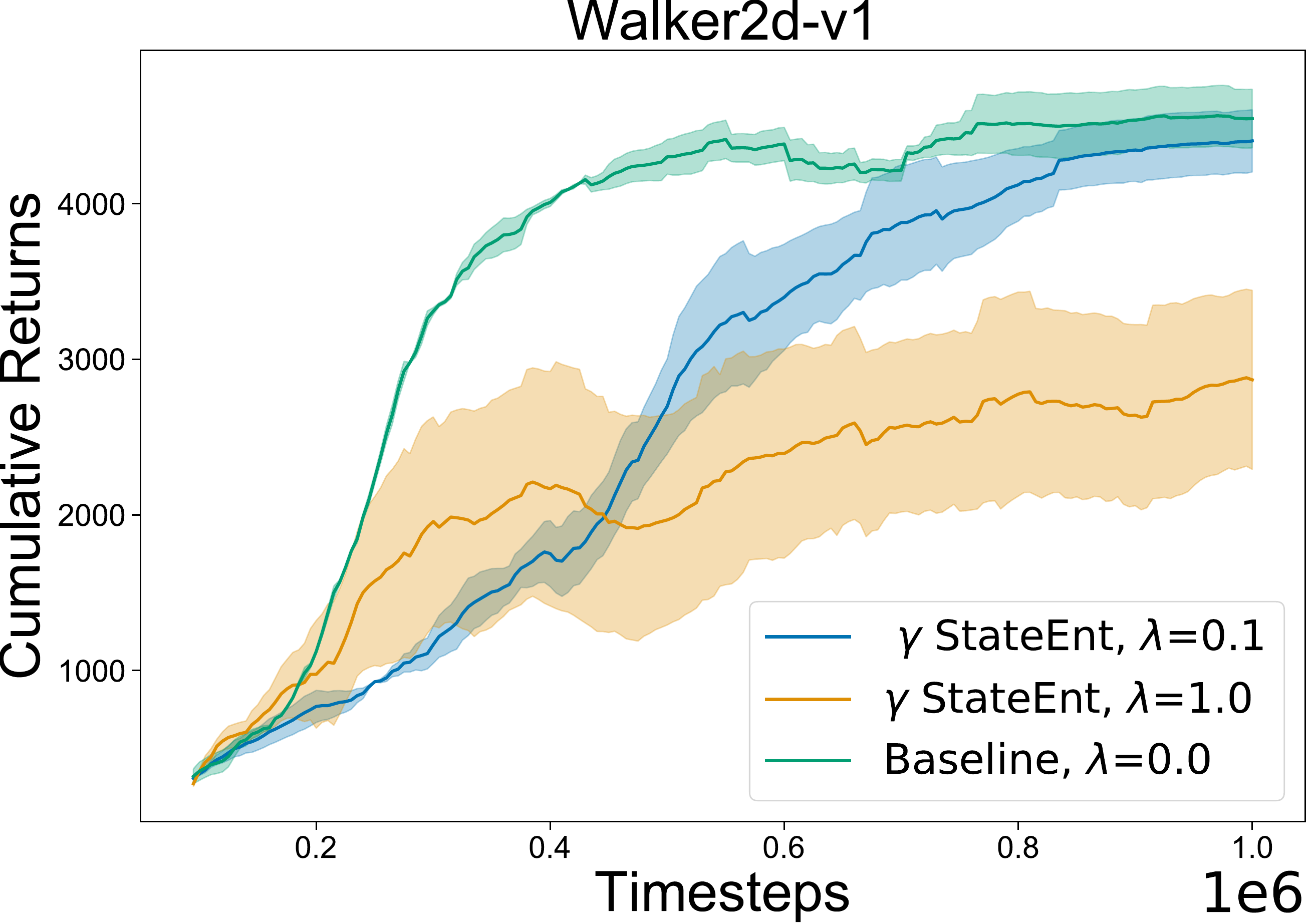}
    \end{subfigure}    
\caption{Additional experimental results with SAC with $\gamma$ StateEnt, for entropy regularization with discounted state distribution. We analyse the significance of using the discounted state distribution over control tasks, with different ranges of $\lambda$ parameters. Note that, for our experiments, we did not do an extensive hyperparameter tuning, but only with the range of $\lambda$ hyperparameters presented in the ablation study here.}
\label{fig:sac_ablation_2}
\end{center}    
\end{figure}

\subsection{Reproducibility Checklist}

We follow the reproducibility checklist from \href{https://www.cs.mcgill.ca/~jpineau/ReproducibilityChecklist.pdf}{Pineau, 2018} and include further details here. For all the models and algorithms, we have included details that we think would be useful for reproducing the results of this work.

\begin{itemize}

\item For all \textbf{models} and \textbf{algorithms} presented, check if you include:

\begin{enumerate}
        
    \item \textit{\textbf{Description of Algorithm and Model}} : 
    We included an algorithm box and provided a description of our algorithm. Our method relies on using a separate density estimation network, (we use a VAE in our work) which takes as input the policy parameters $\theta$ and reconstructs the current state $s_t$. Following this, we train the VAE with the usual variational lower bound (readily available in lot of existing pytorch implementations of VAE). This lower bound acts as the regularizer in our approach, ie, our policy gradient objective directly uses the ELBO as the regulairzer. In the main draft, we provided justification of why this is true. A key step to our implementation (we used pytorch) relies on ensuring that the $.backward()$ update can be used properly, as this the key step for practical implementation of our approach. Please see the code provided along with the draft. 
    
    For the variational auto-encoder based density estimation, we mostly use linear layers with $tanh$ non-linearity in the networks, and use a fixed latent space of size $Z=64$. We use a Gaussian distribution over the latent space $q(Z|S)$ and a unit Gaussian prior for the $KL(q(z|s)||p(z))$. The variational lower bound used as the regularizer is weighted with $\lambda$, where for most of our experiments, we use $\lambda$ in the range of $0.001, 0.01, 0.1$ and $1.0$

    \item \textit{Analysis of Complexity} : We do not include any separate analysis of the complexity of our algorithm. Our method can be used on top of any existing RL algorithms (policy gradient or actor-critic mehtods), where the only extra computation we need is to estimate the variational lower bound for the VAE. This VAE is, however, trained with the same set of sampled states, ie, we do not require separate rollouts and samples for training the density estimator. Therefore, we introduce extra computation in our approach, but the sample complexity remains the same. For comparison with baseline, we use $\lambda=0.0$ for a fair comparison with the baseline.

    \item \textit{Link to downloadable source code} : We provide code for our work in a separate file, for all the experiments used in this work. Furthermore, we provide details of experimental setup below, to ensure our experiments can be exactly reproduced, and additional experimental results and ablation studies are provided in the appendix. 
    
\end{enumerate}

\item For any \textbf{theoretical claim}, check if you include:

\begin{enumerate}
    \item Our key theoretical contribution is based on estimating the state distributions (discounted and stationary) by training a state density estimator. In the main draft, we clearly explain the connections between the variational lower bound objective, required to train the density estimator, and the connection it has to our state entropy regularized objective. 
    
    \item \textit{Complete Proof of Claim :} In appendix, we have also included a clear derivation of our proposed approach, using existing theorem used in the literature, to clarify how exactly our proposed approach differs. We further provide proof for a three-time-scale algorithm, which provides a key justification to our work. 
    
    \item \textit{A clear explanation of any assumptions} : 
    Our key assumption is that, we assume that the stochastic process is fast mixing, and due to erodiciity under all considered policies, the stationary state distribution can be estimated efficiently. We further assume that the varational lower bound used for our objective is tight, and closely approximates the log-likelihood of the visited states $\log p(s)$

\end{enumerate}

\item For all figures and tables that present \textbf{empirical results}, check if you include:

\begin{enumerate}
    \item \textit{Data collection process} : We use our approach on top of any existing policy gradient based approach, and therefore follow the same standard policy rollout based data collection process. Our proposed approach do not require any extra sample complexity, as all the models are trained with the same set or batch of data.

    \item \textit{Downloadable version of environment} : We use open-sourced environment implementations from OpenAI gym in most of the tasks, and the Mujoco control simulator readily available online. We provide code for all our experiments, including code for any additional experiments that were used for justifying the hypothesis of our work. Often these environments come with open-sourced tuned deep RL algorithms, and we used existing open-sourced implementations (links available below) to build our approach on top of existing algorithms.

    \item \textit{Description of any pre-processing step} : We do not require any data pre-processing step for our experiments. 
    
    \item \textit{Sample allocation for training and evaluation} : We use standard RL evaluation framework for our experimental results. In our experiments, as done in any RL algorithm, the trained policy is evaluated at fixed intervals, and the performance is measured by plotting the cumulative returns. In most of our presented experimental results, we plot the cumulative return performance measure. All our experiments for the simple tasks are averaged over $10$ random seeds, and over $10$ random seeds for the deepRL control tasks.

    \item \textit{Range of hyper-parameters considered : }  For our experiments, we did not do any extensive hyperparameter tuning. We took existing implementations of RL algorithms (details of which are given in the Appendix experimental details section below), which generally contain tuned implementations. For our proposed method, we only introduced the extra hyperparameter $\lambda$ state distribution entropy regularizer. We tried our experiments with only 3 different lambda values ($\lambda = 0.001, 0.01 and 0.1$) and compared to the baseline with $\lambda=0.0$ for a fair comparison. Both our proposed method and the baseline contains the same network architectures, and other hyperparameters, that are used in existing open-sourced RL algorithms. We include more details of our experiment setups in the next section in Appendix.
    
    \item \textit{Number of Experiment Runs} : For all our experimental results, we plot results over $10$ random see.Eh of our hyper-parameter tuning is also done with $10$ experiment runs with each hyperparameter.These random seeds are sampled at the start of any experiment, and plots are shown averaged over 10 runs. We note that since a lot of DeepRL algorithms suffer from high variance, we therefore have the high variance region in some of our experiment results.
    
    \item \textit{Statistics used to report results} : In the resulting figures, we plot the mean, $\mu$, and standard error $\frac{\sigma}{\sqrt(N)})$ for the shaded region, to demonstrate the variance across runs and around the mean. We note that some of the environments we used in our experiments, are very challenging to solve (e.g 3D maze navigation domains), resulting in the high variance (shaded region) around the plots. The Mujoco control experiments done in this work have the standard shaded region as expected in the performance in the baseline algorithms we have used (DDPG and SAC).
    
    \item \textit{Error bars} : The error bars or shaded region are due to $std / sqrt(N)$ where $N=10$ for the number of experiment runs.
    
    \item \textit{Computing Infrastrucutre} : We used both CPUs and GPUs in all of our experiments, depending on the complexity of the tasks. For some of our experiments, we could have run for more than $10$ random seeds, for each hyperparameter tuning, but it becomes computationally challenging and a waste of resources, for which we limit the number of experiment runs, with both CPU and GPU to be a standrd of $10$ across all setups. 
\end{enumerate}

\end{itemize}

\subsection{Additional Experimental Details}
\label{sec:ref:detailed_experimental_setup}
In this section, we include further experimental details and setup for the results presented in the paper

\textbf{Experiment setup for State Space Coverage}

The sparse reward gridworld environments are implemented in the open-source package EasyMDP. For this task, we use a parallel threaded \textsc{Reinforce} implementation, and only compare the performance of our proposed approach qualitatively by plotting the state visitation heatmaps.

    


\textbf{Experiment setup in Continuous Control Tasks}

For the continuous control experiments, we used the open-source implementation of DDPG available from the accompanying paper \citep{td3}. We further use a SAC implementation, from a modified implementation of DDPG. Both the implementations of DDPG and SAC are provided with the accompanying codebase. We used the same architectures and hyperparameters for DDPG and SAC as reported in \citep{td3}.

\end{document}